%% file: approx.tex
\documentclass[11pt]{article}
\usepackage{amsmath}
\usepackage{amssymb}
\usepackage{amsthm}
\usepackage{amsfonts}
\usepackage{amsbsy}
\usepackage{authblk}
\usepackage{mathrsfs}
\usepackage[english]{babel}
\usepackage{epsfig}
\usepackage{fullpage}
\usepackage[T1]{fontenc}
\usepackage{graphicx}
\usepackage[colorlinks=true, allcolors=blue]{hyperref}
\usepackage[utf8x]{inputenc}
\usepackage{multirow}
\usepackage{multicol}
\usepackage[numbers]{natbib}

\newcommand{\cO}{\mathcal{O}}
\newcommand{\EE}{\mathbb{E}}

\newcommand{\indic}{\mathbf{1}}
\newcommand{\yvec}{{\bf y}}

\newcommand{\Wapprox}{\textsc{WFA-Approx}}
\newcommand{\Wsampleapprox}[1]{\textsc{WFA-SampleApprox}$($#1$)$}
\newcommand{\Wsamplekatz}[1]{\textsc{WFA-SampleKatz}$($#1$)$}

\newcommand{\ignore}[1]{}

\newcommand{\bos}{$\wedge\,$}
\newcommand{\eos}{\text{$\$\,$}}
\newcommand{\cL}{\mathcal{L}}
\newcommand{\sP}{\mathscr{P}}

\newtheorem{lemma}{Lemma}

\newtheorem{corollary}{Corollary}

\def\Rset{\mathbb{R}}
\def\Rsetp{\Rset_+}
\def\argmin{\mathop{\rm argmin}\limits}
\def\argmax{\mathop{\rm argmax}\limits}

\newcommand{\src}[1]{p[{#1}]}
\newcommand{\dst}[1]{n[{#1}]}
\newcommand{\lab}[1]{\ell[#1]}
\newcommand{\weight}[1]{w[#1]}
\newcommand{\lastlab}[1]{\lab{#1_{|#1|}}}
\newcommand{\weighted}[1]{\hat{#1}}
\newcommand{\best}[1]{\tilde{#1}}

\title{Approximating probabilistic models as weighted finite automata}

\author{Ananda Theertha Suresh, Brian Roark, Michael Riley ,Vlad
  Schogol \\ \texttt{$\{$ theertha, roark, riley, vlads
    $\}$@google.com} \\ Google Inc. }

\begin{document}

\maketitle

\input{approx_abstract}

\section{Introduction}
\label{sec:intro}
\input{approx_intro}

\section{Related work}
\label{sec:prevwork}
\input{approx_prevwork}

\section{Theoretical analysis}
\label{sec:theory} 
\input{approx_theory}

\section{Algorithms}
\label{sec:algo} 
\input{approx_algo}

\section{Experiments}
\label{sec:expt}
\input{approx_experiments}

\section{Software library}
\label{sec:lib}
\input{approx_lib}

\section{Summary and discussion}
\label{sec:discuss}
\input{approx_discuss}

\bibliographystyle{plainnat}
\bibliography{approx}

\end{document}

%% file: approx_abstract.tex
\begin{abstract}
Weighted finite automata (WFA) are often used to represent
probabilistic models, such as $n$-gram language models, since, among other things, they are
efficient for recognition tasks in time and space.
The probabilistic source to be represented as a WFA, however, may come in many
forms. Given a generic probabilistic model over sequences, we propose an
algorithm to approximate it as a weighted finite automaton such that
the Kullback-Leibler divergence between the source model and the WFA target 
model is minimized. The proposed algorithm involves a counting step
and a difference of convex optimization step, both of which can be
performed efficiently. We demonstrate the usefulness of our approach
on various tasks, including distilling $n$-gram models from neural
models, building compact language models, and building open-vocabulary
character models.
The algorithms used for these experiments are available in an
open-source software library.
\end{abstract}

%% file: approx_intro.tex
Given a sequence of symbols $x_1,x_2,\ldots, x_{n-1}$, where symbols
are drawn from the alphabet $\Sigma$, a probabilistic model $S$ assigns
probability to the next symbol $x_n \in \Sigma$ by
\[
 p_s[x_{n} | x_{n-1} \ldots x_1].
\]
Such a model might be Markovian,
where
\[
p_s[x_n | x_{n-1} \ldots x_1] = p_s[x_n | x_{n-1} \ldots x_{n-k+1}],
\]
such as a $k$-gram language model (LM) \cite{Chen98} or it might be
non-Markovian such as a long short-term memory (LSTM) neural network
LM \cite{sundermeyer2012lstm}. Our goal is to approximate
a probabilistic model as a {\em weighted finite automaton} (WFA)
such that the weight assigned by the WFA is close to the probability
assigned by the source model. Specifically, we will seek to minimize
the Kullback-Leibler (KL) divergence between the source $S$ and the
target WFA model.

Representing the target model as a WFA has many advantages including
efficient use, compact representation, interpretability, and
composability. WFA models have been used in many applications including
speech recognition
\cite{MohriPereiraRiley2008}, speech synthesis \cite{rws2015}, optical
character recognition \cite{breuel2008}, machine translation
\cite{iglesias2011}, computational biology \cite{Durbin1998}, and
image processing \cite{AlbertKari2009}. One particular problem of
interest is language modeling for on-device (virtual) keyboard
decoding \cite{ouyang2017mobile}, where WFA models are widely used due
to space and time constraints. One example use of methods we
present in this paper comes from this domain, involving approximation of
a neural LM. Storing training data from the
actual domain of use (individual keyboard activity) in a
centralized server and training $k$-gram or WFA models directly may
not be feasible due to privacy constraints \cite{hard2018federated}.
To circumvent this, an LSTM model can be trained by federated
learning \cite{konevcny2016federated, mcmahan2017communication, hard2018federated},
converted to a WFA at the server, and then used
for fast on-device inference. This not only may improve performance
over training the models just on out-of-domain publicly available data, but also
benefits from the additional privacy provided by federated learning.
\citet{chen2019federated} use the methods presented in this
paper for this very purpose.

There are multiple reasons why one may choose to approximate a source
model with a WFA. One may have strong constraints on system latency,
such as the virtual keyboard example above. Alternatively, a
specialized application may require a distribution over just a subset
of possible strings, but must estimate this distribution from a more
general model -- see the example below regarding utterances for
setting an alarm.  To address the broadest range of use cases, we aim
to provide methods that permit large classes of source models and
target WFA topologies.  We explore several distinct scenarios
experimentally in Section \ref{sec:expt}.

Our methods allow {\em failure transitions} \cite{aho1975, mohri1997} in the
target WFA, which are taken only when no immediate match is possible at
a given state, for compactness. For example, in the WFA representation
of a backoff $k$-gram model, failure transitions can compactly
implement the backoff
\cite{katz,Chen98,Allauzen03,novak2013,hellsten17}. The inclusion of failure
transitions complicates our analysis and algorithms but is highly
desirable in applications such as keyboard decoding. Further, to
avoid redundancy that leads to inefficiency, we assume the
target model is {\em deterministic}, which requires that at each state
there is at most one transition labeled with a given symbol.

The approximation problem can be divided into two steps: (1) select an
unweighted automaton $A$ that will serve as the {\em topology} of the
target automaton and (2) weight the automaton $A$ to form our weighted
approximation $\weighted{A}$. The main goal of this paper is the latter
determination of the automaton's weighting in the approximation.  If
the topology is not known, we suggest a few techniques for inferring
topology later in the introduction.

We will now give some simple topology examples to illustrate the
approximation idea. In Section \ref{sec:expt} we will give
larger-scale examples.  Consider the unweighted automaton $A$ in
Figure~\ref{fig:alarm} that was designed for what one might say
to set an alarm.  To use this in an application such as speech
recognition, we would want to weight the automaton with some
reasonable probabilities for the alternatives.  For example, people
may be more likely to set their alarms for six or seven than four.
In the absence of data
specifically for this scenario, we can fall back on some available
background LM $M$, trained on a large suitable corpus. In
particular, we can use the conditional distribution

\begin{equation}
p_M[x_1 \ldots x_n | x_1 \ldots x_n \in {\cal L}(A)] =
\frac{p_M[x_1 \ldots x_n] 1_{x_1 \ldots x_n \in {\cal L}(A)}}
{\sum_{x_1 \ldots x_n \in {\cal L}(A)} p_M[x_1 \ldots x_n]},
\label{eq:restrict-lm}
\end{equation}
where $1$ is the indicator function and ${\cal L}(A)$
is the regular language accepted by the automaton $A$, as our source 
distribution $S$. We then use the unweighted automaton $A$ as our 
target topology. 

If $M$ is represented as a WFA, our approximation will
in general give a different solution than forming
the finite-state intersection with $A$ and
{\em weight-pushing} to normalize the result
\cite{MohriPereiraRiley2008,Mohri2009}.
Our approximation has the same states as $A$ 
whereas weight-pushed $M \cap A$ has $O(|M||A|)$ states.
Furthermore, weight-pushed $M \cap A$ is an exact
WFA representation of the distribution in Equation~\ref{eq:restrict-lm}.
\begin{figure}[t]
\begin{center}
\includegraphics[scale=.55]{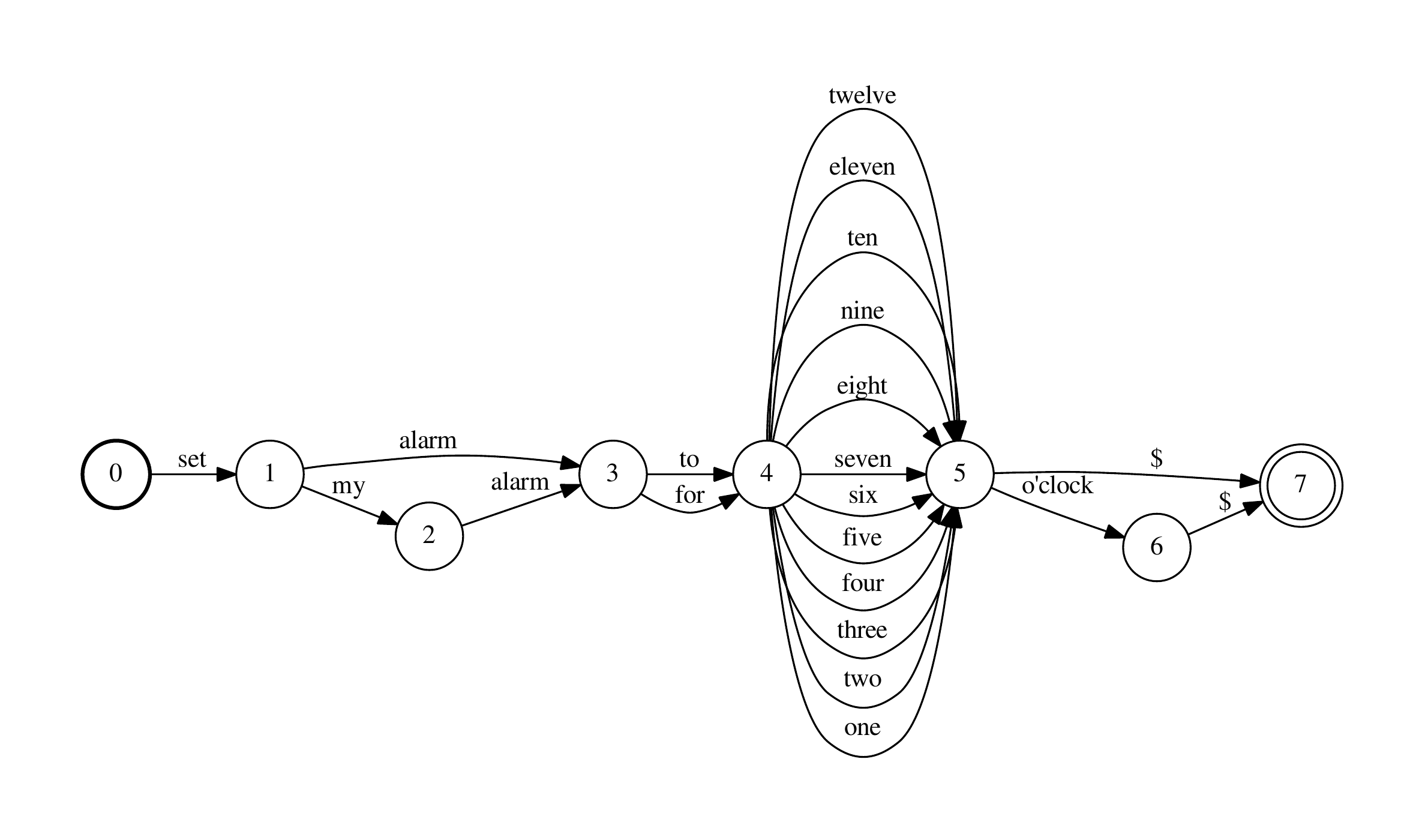}
\end{center}
\caption{An unweighted automaton that specifies what one might say 
to set an alarm.
The initial state is the bold circle and the final state is the double circle.
By convention, we terminate all accepted strings with the symbol \eos.}
\label{fig:alarm}
\end{figure}

As stated earlier, some applications may simply require smaller models
or those with lower latency of inference, and in such scenarios the
specific target topology may be unknown. In such cases, one
choice is to build a $k$-gram deterministic finite automaton (DFA)
topology from a corpus drawn from $S$ \cite{Allauzen03}. This could
be from an existing corpus or from random samples drawn from
$S$. Figure~\ref{fig:top}a shows a trigram topology for the very
simple corpus $aab$. Figure~\ref{fig:top}b shows an alternative
topology that allows {\em skip}-grams. Both of these representations
make use of failure transitions. These allow modeling strings unseen
in the corpus (e.g., $abab$) in a compact way by failing or {\em backing-off} to
states that correspond to lower-order histories. Such models can be
made more elaborate if some transitions represent classes, such as names
or numbers, that are themselves represented by sub-automata. As mentioned
previously, we will mostly assume we have a topology either pre-specified or
inferred by some means and focus on how to weight that topology to
best approximate the source distribution.
We hence focus our evaluation on intrinsic measures of model
quality such as perplexity or bits-per-character.\footnote{Model size
or efficiency of inference may be a common motivation for the model
approximations we present, which may suggest a demonstration of the
speed/accuracy tradeoff for some downstream use of the models.
However, as stated earlier in this introduction, there are mulitple reasons why 
an approximation may be needed, and our goal in this paper is to
establish that our methods provide a well-motivated approach should an
approximation be required for any reason.}

\begin{figure}[t]
\begin{center}
\begin{tabular}{c}
\vspace{-0.5cm}
\includegraphics[scale=.55]{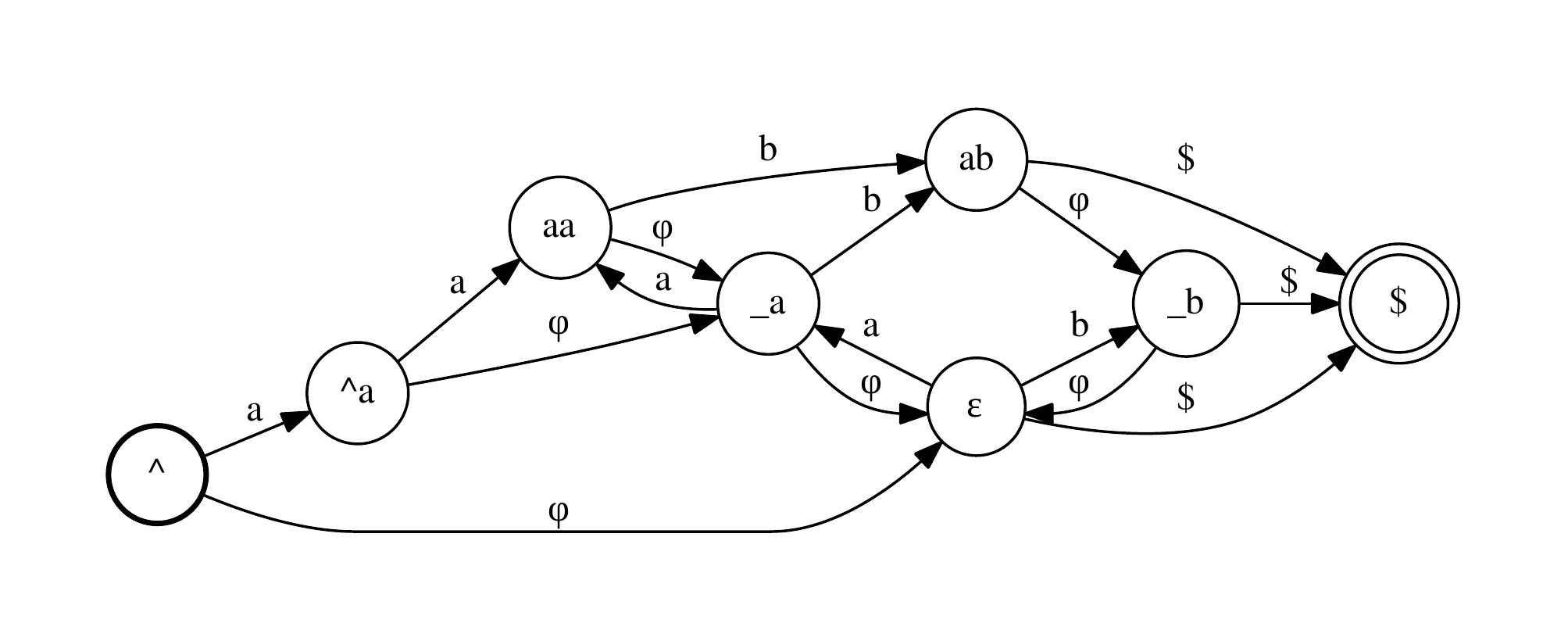}\\
\vspace{-0.75cm}
(a)\\
\vspace{-0.5cm}
\includegraphics[scale=.55]{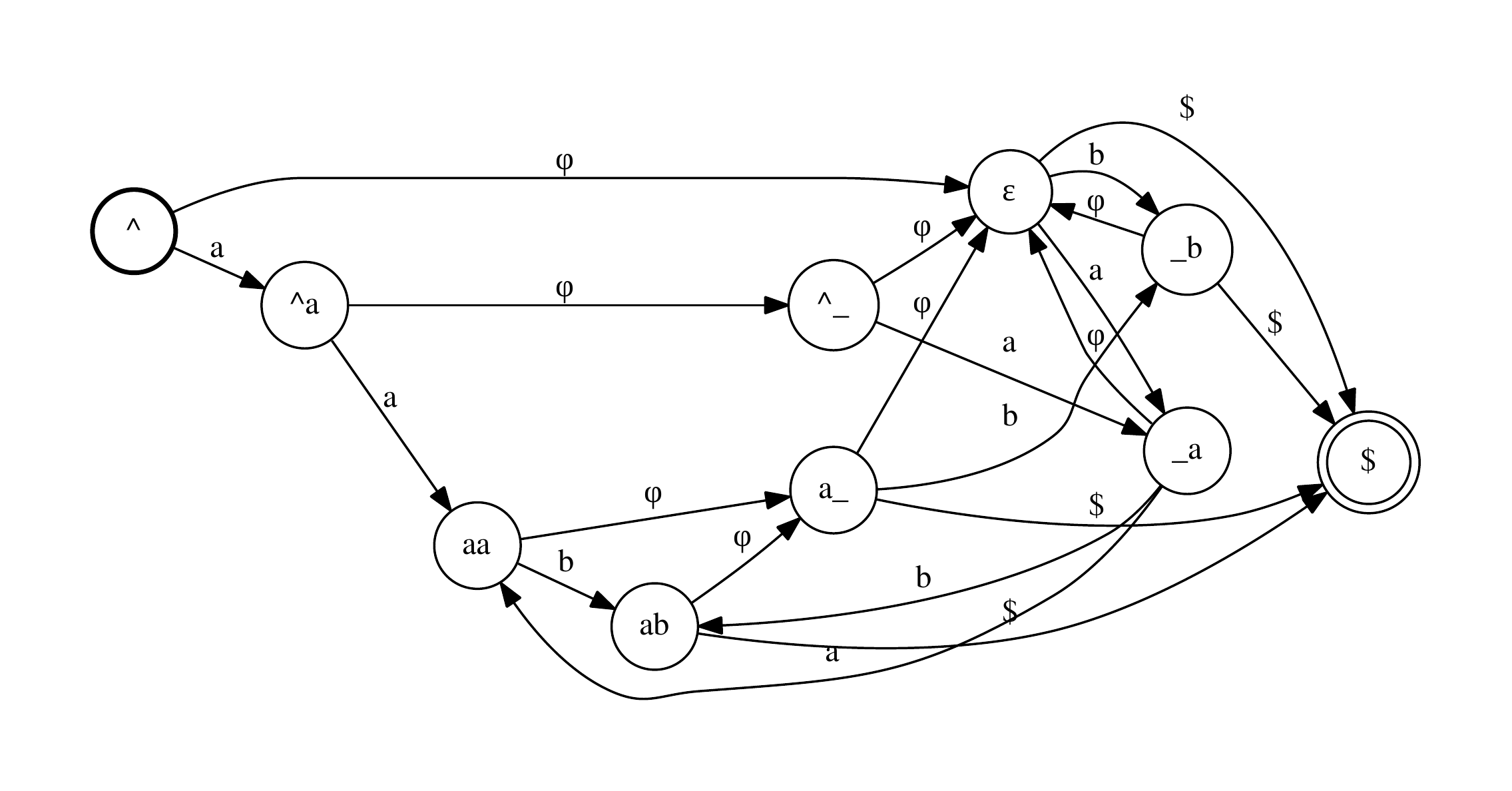}\\
\vspace{-0.5cm}
(b)
\end{tabular}
\end{center}
\caption{Topology examples derived from the corpus $aab$. States
are labeled with the context that is remembered,
\bos denotes the
initial context, $\epsilon$ the empty context,
\eos the final context (and terminates
accepted strings), and $\_$ matches any symbol in a context.
(a) $3$-gram topology:
  failure transitions, labeled with $\varphi$,
  implement backoff from histories $xy$ to $\_y$ to $\epsilon$.
  (b) {\em skip}-gram topology: failure transitions implement backoff
  instead from histories $xy$ to $x\_$.}
\label{fig:top}
\end{figure}

This paper expands upon an earlier, shorter version \cite{suresh2019fsmnlp} by
also providing, beyond the additional motivating examples that have
already been presented in this introduction: an extended related work
section and references; expansion of the theoretical analysis from
three lemmas without proofs (omitted for space) to five lemmas (and a
corollary) with full proofs; inclusion of additional algorithms for
counting (for general WFA source and target in addition to
$\varphi$-WFA); inclusion of additional experiments, including some
illustrating the exact KL-minimization methods available for WFA
sources of different classes; and documentation of the open-source
library that provides the full functionality presented here.

The paper is organized as follows. In Section~\ref{sec:prevwork} we
review previous work in this area. In Section~\ref{sec:theory} we give
the theoretical formulation of the problem and the minimum KL
divergence approximation. In Section~\ref{sec:algo} we present
algorithms to compute that solution. One algorithm is for the case
that the source itself is finite-state. A second algorithm is for the
case when it is not and involves a sampling approach. In
Section~\ref{sec:expt} we show experiments using the
approximation. In Section~\ref{sec:lib} we briefly describe the open-source
software used in our experiments. Finally, in Section \ref{sec:discuss} we 
discuss the results and offer conclusions.

%% file: approx_prevwork.tex
In this section we will review methods both for inferring unweighted
finite-state models from data and estimating the weight distribution
as well in the weighted case. We start with the unweighted case.

There is a long history of unweighted finite-state model inference
\cite{parekh2000, cicchello2003}.
\citet{gold1967} showed that an arbitrary regular set $L$
cannot be learned, {\em identified in the limit}, strictly from the 
presentation of a sequence of positive examples that eventually includes
each string in $L$. This has led to several alternative lines of attack.

One approach is to include the negative examples in the sequence. Given
such a {\em complete sample}, there are polynomial-time algorithms that
identify a regular set in the limit \cite{gold1978}.
For example, a {\em prefix tree} of the positive examples can be built
and then states can be merged so long as they do not cause a negative example 
to be accepted \cite{oncina1992, dupont1996}. Another approach is to train 
a recurrent neural network (RNN)
on the positive and negative examples and then extract a finite automaton 
by quantizing the continuous state space of the RNN \cite{giles1992,
jacobsson2005}.

A second approach is to assume a teacher is available 
that determines not only if a string is a positive or negative example
but also if the language of the current hypothesized automaton equals $L$
or, if not, provides a counterexample. In this case
the minimal $m$-state DFA corresponding to $L$ can be learned in
time polynomial in $m$ \cite{angluin1987}. \citet{weiss2018extracting}
apply this method for DFA extraction from an RNN.

A third approach is to assume a probability distribution over the
(positive only) samples. With some reasonable restrictions on the
distribution, such as that the probabilities are generated from a weighted
automaton $A$ with $L = {\cal L}(A)$, then $L$ is identifiable in the limit
with `high probability' \cite{angluin1988, pitt1989}. 

There have been a variety of approaches for estimating
weighted automata. A variant of the prefix tree construction 
can be used that merges states 
with sufficiently similar suffix distributions,
estimated from source frequencies
\cite{carrasco1994, carrasco1999}.
Approaches that produce (possibly highly) non-deterministic results
include the Expectation-Maximization (EM)
algorithm \cite{dempster1977} applied to a fully connected Hidden
Markov models or spectral methods applied to automata \cite{Balle2012,
Balle2014}. \citet{eisner2001} describes an algorithm for
estimating probabilities in a finite-state {\em transducer} from data
using EM-based methods.  \citet{weiss2019learning}
and \citet{okudono2019weighted} provide adaptations to
the \citet{weiss2018extracting} DFA extraction algorithm to yield
weighted automata.

For approximating neural network (NN) models as WFAs,
\citet{deoras2011variational} used an RNN LM to generate samples
that they then used to train a $k$-gram LM.
\citet{arisoy2014converting} used deep neural network (DNN) models
of different orders to successively build and prune a $k$-gram LM with
each new order constrained by the previous order.
\citet{adel2014comparing} also trained DNNs of different orders,
built a $k$-gram LM on the same data to obtain a topology and then
transferred the DNN probabilities of each order onto that $k$-gram
topology. \citet{tino1997} quantized the continuous
state space of an RNN and then estimated the transition probabilities
from the RNN.  \citet{lecorve2012conversion}
quantized the hidden states in an LSTM to form a finite-state model
and then used an entropy criterion to backoff to low-order $k$-grams
to limit the number of transitions. See
Section~\ref{sec:prev_work_discussion} for a more detailed comparison
with the most closely related methods, once the details of our
algorithms have been provided.

Our paper is distinguished in several respects from previous
work. First, our general approach does not depend on the form the
source distribution although we specialize our algorithms for (known)
finite-state sources with an efficient direct construction and for
other sources with an efficient sampling approach.  Second, our targets are a
wide class of deterministic automata with failure transitions.  These
are considerably more general than $k$-gram models but retain the
efficiency of determinism and the compactness failure transitions
allow, which is especially important in applications with large alphabets like
language modeling. Third, we show that our approximation searches for
the minimal KL divergence between the source and target distributions,
given a fixed target topology provided by the application or some
earlier computation.

%% file: approx_theory.tex
\subsection{Probabilistic models}
\label{subsec:probmod}
Let $\Sigma$ be a finite alphabet.
Let $x^{n}_i \in \Sigma^*$ denote the string $x_i
x_{i+1}\ldots x_n$ and $x^n \triangleq x^n_1$.
A probabilistic model $p$ over $\Sigma$ is
a probabilistic distribution over the next symbol $x_n$, given the previous
symbols $x^{n-1}$, such that\footnote{We define $x^0 \triangleq \epsilon$, the empty string, and adopt $p(\epsilon) = 0$.}
\[
\sum_{x \in \Sigma} p(x_n = x | x^{n-1}) = 1 \text{ and } p(x_n = x
| x^{n-1}) \geq 0, \forall x \in \Sigma.
\]
Without loss of generality, we assume that the model maintains an
internal state $q$ and updates it after observing the next
symbol.\footnote{In the most general case, $q(x^n) =
  x^n$.} Furthermore, the probability of the subsequent state just
depends on the state $q$
\[
p(x^n_{i+1} | x^i) = p(x^n_{i+1}| q(x^i)),
\]
for all $i, n, x^i, x^n_{i+1}$, where $q(x^{i})$ is the state that the
model has reached after observing sequence $x^i$.  Let $Q(p)$ be the
set of possible states. Let the language
$\cL(p) \subseteq \Sigma^*$ defined by the distribution $p$ be
\begin{equation}
\cL(p) \triangleq \{x^n \in \Sigma^* : p(x^n) > 0 \text{ and }
x_n = \eos \text{ and }
x_i \neq \eos, \forall\ i < n\}.
\label{eq:language}
\end{equation}
The symbol $\eos$ is used as a stopping criterion. Further for all
$x^n \in \Sigma^*$ such that $x_{n-1}=\eos$, $p(x_n | x^{n-1}) = 0$.

The KL divergence between the source model $p_s$ and the target model
$p_a$ is given by
\begin{equation}
D(p_s || p_a) = \sum_{x^n \in \Sigma^*} p_s(x^n) \log \frac{p_s(x^n)}{p_a(x^n)},
\end{equation}
where for notational simplicity, we adopt the notion $0/0=1$ and $0
\log (0/0) = 0$ throughout the paper.
Note that for the KL divergence to
be finite, we need $\cL(p_s) \subseteq \cL(p_a)$.
We will assume throughout
the source entropy $H(p_s) = - \sum_{x^n} p_s(x^n) \log p_s(x^n)$
is finite.\footnote{If $|Q(p_s)|$ is finite, it
can be shown that $H(p_s)$ is necessarily finite.}
We first reduce the
KL divergence between two models as follows \cite[cf.][]{carrasco1997,cortes2008}.
In the following, let $q_*$ denote the states of the probability distribution $p_*$.
\begin{lemma}
  If $\cL(p_s) \subseteq \cL(p_a)$, then
  \begin{equation}
   \label{eq:kl1}
    D(p_s ||p_a) = \sum_{q_a \in Q_a} \sum_{q_s \in Q_s} \sum_{x\in
      \Sigma} \gamma(q_s,q_a)\, p_s(x|q_s) \log
    \frac{p_s(x|q_s)}{p_a(x|q_a)},
    \end{equation}
  where
  \begin{equation}
\gamma(q_s, q_a) = \sum^\infty_{i=0}
\sum_{x^i: q_s(x^i)=q_s, q_a(x^i)=q_a} p_s(x^i).
    \end{equation}
  \label{thm:kldiv}
  \end{lemma}
\begin{proof}
  \begin{align*}
    D(p_s || p_a) &= \sum^{\infty}_{n=1}\sum_{x^n} p_s(x^n) \log
    \frac{p_s(x^n)}{p_a(x^n)} \\
    &= \sum^\infty_{n=1}\sum_{x^n}
    p_s(x^n) \sum^n_{i=1} \log
    \frac{p_s(x_i|x^{i-1})}{p_a(x_i|x^{i-1})}
    \\ &= \sum^\infty_{n=1}
    \sum^{n}_{i=1}\sum_{x^n} p_s(x^n) \log
    \frac{p_s(x_i|x^{i-1})}{p_a(x_i|x^{i-1})}
\\ &= \sum^\infty_{n=1}
    \sum^{n}_{i=1}\sum_{x^n} p_s(x^{i-1})p_s(x_{i}|x^{i-1})
    p_s(x^{n}_{i+1}|x^{i}) \log
    \frac{p_s(x_i|x^{i-1})}{p_a(x_i|x^{i-1})} \\ &= \sum^\infty_{i=1}
    \sum_{x^{i-1}} p_s(x^{i-1})\sum_{x_i}p_s(x_{i}|x^{i-1}) \log
    \frac{p_s(x_i|x^{i-1})}{p_a(x_i|x^{i-1})} \cdot \sum_{n\geq i}
    \sum_{x^n_{i+1}} p_s(x^{n}_{i+1}|x^{i}) \\ &= \sum^\infty_{i=1}
    \sum_{x^{i-1}} p_s(x^{i-1})\sum_{x_i}p_s(x_{i}|x^{i-1}) \log
    \frac{p_s(x_i|x^{i-1})}{p_a(x_i|x^{i-1})}.
    \end{align*}
  By definition, the probability of the next symbol conditioned on the past
  just depends on the state. Hence grouping terms corresponding to
  same states both in $s$ and $t$ yields,
  \begin{align*}
& \sum^\infty_{i=1} \sum_{x^{i-1}} p_s(x^{i-1})
    \sum_{x_i}p_s(x_{i}|x^{i-1}) \log
    \frac{p_s(x_i|x^{i-1})}{p_a(x_i|x^{i-1})} \\ & = \sum^\infty_{i=1}
    \sum_{x^{i-1}} p_s(x^{i-1}) \sum_{x_i}p_s(x_{i}|q_s(x^{i-1})) \log
    \frac{p_s(x_i|q_s(x^{i-1}))}{p_a(x_i|q_a(x^{i-1}))} \\ & =
    \sum_{q_a \in Q_a} \sum_{q_s \in Q_s} \sum^\infty_{i=1}
    \sum_{x^{i-1}: q_s(x^{i-1})=q_s, q_a(x^{i-1})=q_a} p_s(x^{i-1})
    \sum_{x_i}p_s(x_{i}|q_s) \log \frac{p_s(x_i|q_s)}{p_a(x_i|q_a)}
    \\ & = \sum_{q_a \in Q_a} \sum_{q_s \in Q_s} \sum_{x_i}
    \gamma(q_s, q_a) \, p_s(x_i|q_s)  \log \frac{p_s(x_i|q_s)}{p_a(x_i|q_a)}.
  \end{align*}
  Replacing $x_i$ by $x$ yields the lemma.
  \end{proof}

The quantity $\gamma(q_s, q_a)$ counts each string $x^i$
that reaches both state $q_s$ in $Q_s$ and
state $q_a$ in $Q_a$ weighted by its probability according
to $p_s$.
Equivalently it counts each string $x^i$
reaching state $(q_s, q_a)$ in $Q_s \times Q_a$ (an interpretation
we develop in Section~\ref{sec:wfa_source}).\footnote{
  $\gamma$ is not (necessarily) a probability distribution
  over $Q_s \times Q_a$. In fact, $\gamma(q_s, q_a)$ can be greater than $1$.}
Note $\gamma$ does not depend on distribution $p_a$.

The following corollary is useful for finding the probabilistic
model $p_a$ that has the minimal KL divergence from a model
$p_s$.

\begin{corollary}
  Let $\sP$ be a set of probabilistic models $p_a$
for which $\cL(p_s) \subseteq \cL(p_a)$.
Then
  \begin{equation}
    \argmin_{p_a \in \sP} D(p_s ||p_a)  \nonumber \\
= \argmax_{p_a \in \sP}
\sum_{q_a \in Q_a} \sum_{x\in \Sigma} c(x,q_a) \log p_a(x|q_a), \label{eq:kl2}
    \end{equation}
  where
  \begin{equation}
 c( x, q_a) =
     \sum_{q_s \in Q_s} \gamma(q_s, q_a) \, p_s(x|q_s).
    \label{eq:cfactor0}
    \end{equation}
  \label{thm:kldivmin}
  \end{corollary}
\begin{proof}
By Lemma~\ref{thm:kldiv}
\begin{align}
\argmin_{p_a \in \sP} D(p_s || p_a)
= & \argmin_{p_a \in \sP}
\sum_{q_a \in Q_a} \sum_{q_s \in Q_s} \sum_{x\in
      \Sigma} \gamma(q_s,q_a)\, p_s(x|q_s) \log
    \frac{p_s(x|q_s)}{p_a(x|q_a)} \nonumber \\
= & \argmin_{p_a \in \sP} \Big \{
\sum_{q_a \in Q_a} \sum_{q_s \in Q_s} \sum_{x\in
      \Sigma} \gamma(q_s,q_a)\, p_s(x|q_s) \log p_s(x|q_s) \nonumber \\ &
\qquad\quad - \sum_{q_a \in Q_a} \sum_{q_s \in Q_s} \sum_{x\in
      \Sigma} \gamma(q_s,q_a)\, p_s(x|q_s)
    \log {p_a(x|q_a)} \Big \} \label{eq:minterms} \\
= & \argmax_{p_a \in \sP} \sum_{q_a \in Q_a} \sum_{x \in \Sigma}
    c(x, q_a) \log p_a(x|q_a), \nonumber
\end{align}
since the first term in Equation~\ref{eq:minterms} does not depend on $p_a$.
\end{proof}
The quantity $c(x_i, q_a)$ counts each string $x^i$ that reaches
a state $(q_s, q_a)$ in $Q_s \times \{q_a\}$ by $x_1^{i-1}$ weighted
by its probability according to $p_s$ (cf. Equation \ref{eq:transcnt2}).

\subsection{Weighted finite automata}
\label{subsec:wfa}
A weighted finite automaton $A = (\Sigma, Q, E, i, f)$
over $\Rsetp$ is given by a finite alphabet
$\Sigma$, a finite set of states $Q$,
a finite set of transitions
$E \subseteq Q \times \Sigma \times \Rsetp \times Q$,
an initial state $i \in Q$ and a final state $f \in Q$.
A transition $e = (\src{e}, \lab{e}, \weight{e}, \dst{e}) \in E$ represents
a move from a {\em previous} (or {\em source}) state
$\src{e}$ to the {\em next} (or {\em destination}) state $\dst{e}$ with
the {\em label} $\lab{e}$ and {\em weight} $\weight{e}$. The
transitions with previous state $q$ are denoted by $E[q]$ and the
labels of those transitions as $L[q]$.

A deterministic WFA has at most one transition with a given label leaving
each state.
An {\em unweighted (finite) automaton} is a WFA that satisfies
$\weight{e} = 1, \forall e \in E$.
A {\em probabilistic (or stochastic) WFA} satisfies
\[
\sum_{e \in E[q]} \weight{e} = 1 \text{ and } \weight{e} \geq 0,
\quad\forall q \in Q - \{f\}.
\]

Transitions $e_1$ and $e_2$ are {\em consecutive}
if $\dst{e_i} = \src{e_{i+1}}$.
A path $\pi = e_1 \cdots e_n \in E^*$ is a finite sequence of consecutive
transitions.
The previous state of a path we denote by $\src{\pi}$ and the
next state by $\dst{\pi}$.
The label of a path is the concatenation of its transition labels
$\lab{\pi} = \lab{e_1} \cdots \lab{e_n}$.
The weight of a path is obtained by multiplying
its transition weights
$\weight{\pi} = \weight{e_1} \times \cdots \times \weight{e_n}$.
For a non-empty path, the $i$-th transition is denoted by
$\pi_i$.

$P(q,q')$ denotes the set of all paths in $A$ from state $q$ to $q'$.
We extend this to sets in the obvious way:
$P(q, R)$ denotes the set of all paths from state
$q$ to $q' \in R$ and so forth.
A path $\pi$ is successful if it is in $P(i, f)$ and in that case
the automaton is said to accept
the input string $\alpha = \lab{\pi}$.

The language accepted by an automaton $A$ is the regular set $\cL(A) =
\{\alpha \in \Sigma^* : \alpha=\lab{\pi}, \pi \in P(i,f) \}$.  The
weight of $\alpha \in \cL(A)$ assigned by the automaton is $A(\alpha)
= \Sigma_{\pi \in P(i, f) : \, \lab{\pi} = \alpha} \weight{\pi}$.
Similar to Equation \ref{eq:language}, we assume a symbol $\eos \in
\Sigma$ such that
\[
\cL(A) \subseteq \{x^n \in \Sigma^* : x_n = \eos \text{ and } x_i \neq \eos, \forall\ i < n\}.
\]
Thus all successful paths are terminated by the symbol $\eos$.

For a symbol $x \in \Sigma$ and a state $q \in Q$ of a deterministic,
probabilistic WFA A, define a distribution $p_a(x | q) \triangleq w
\text { if } (q, x, w, q') \in E$ and $p_a(x|q) \triangleq 0$
otherwise.  Then $p_a$ is a probabilistic model over $\Sigma$ as
defined in the previous section. If $A = (\Sigma, Q, E, i, f)$ is an
unweighted deterministic automaton, we denote by $\sP(A)$ the set of
all probabilistic models $p_a$ representable as a weighted WFA
$\weighted{A} = (\Sigma, Q, \weighted{E}, i, f)$ with the same
topology as $A$ where $\weighted{E} = \{(q,x,p_a(x|q),q') : (q,x,1,q')
\in E \}$.

Given an
unweighted deterministic automaton $A$,
our goal is to find the target distribution
$p_a \in \sP(A)$ that has the minimum KL
divergence from our source probability model $p_s$.
\begin{lemma}
If $\cL(p_s) \subseteq \cL(A)$, then
 \begin{equation}
\argmin_{p_a \in \sP(A)} D(p_s || p_a) = \best{p}(x|q_a)
\triangleq \frac{c(x, q_a)}{c(q_a)},
\label{eq:marginal}
\end{equation}
where
\[
c(q_a) = \sum_{x \in \Sigma} c(x, q_a).
\]
\label{thm:exact}
\end{lemma}
\begin{proof}
From Corollary~\ref{thm:kldivmin}
\begin{align*}
\argmin_{p_a \in \sP(A)} D(p_s ||p_a) & =
\argmax_{p_a \in \sP(A)}
\sum_{q_a \in Q_a} \sum_{x\in \Sigma} c(x,q_a) \log p_a(x|q_a)\\
& = \argmax_{p_a \in \sP(A)}
\sum_{q_a \in Q_a}
\sum_{x\in \Sigma}
c(q_a)\,\best{p}(x|q_a) \log p_a(x|q_a)\\
& = \argmin_{p_a \in \sP(A)}
\sum_{q_a \in Q_a} c(q_a)
\Big \{- \sum_{x\in \Sigma} \best{p}(x|q_a) \log p_a(x|q_a) \Big \}.
\end{align*}
The quantity in braces is minimized when
$p_a(x | q_a) = \best{p}(x|q_a)$ since it is the cross entropy
between the two distributions.
Since $\cL(p_s) \subseteq \cL(A)$, it follows that $\best{p} \in \sP(A)$.
  \end{proof}

\subsection{Weighted finite automata with failure transitions}
\label{subsec:phiwfa}
A {\em weighted finite automaton with failure transitions}
({\em $\varphi$-WFA})
$A = (\Sigma, Q, E, i, f)$ is a WFA extended to
allow a transition to have
a special {\em failure} label denoted by $\varphi$. Then
$E \subseteq Q \times (\Sigma \cup \{\varphi\}) \times \Rsetp \times Q$.

A $\varphi$ transition does not add to a path label; it consumes
no input.\footnote{In other words, a $\varphi$ label, like an $\epsilon$
label, acts as an identity element in string concatenation
\cite{allauzen2018}.}.
However it is followed only when the input cannot be read immediately.
Specifically, a path $e_1\cdots e_n$ in a $\varphi$-WFA is {\em disallowed}
if it contains a subpath $e_i \cdots e_j$ such that
$\lab{e_k} = \varphi$ for all $k$, $i \leq k < j$, and there is another transition $e \in E$
such that $\src{e_i} = \src{e}$ and $\lab{e_j} = \lab{e} \in \Sigma$
(see Figure~\ref{fig:forbid}). Since the label $x = l[e_j]$ can be
read on $e$,
we do not follow the failure transitions to read it on $e_j$ as well.

\begin{figure}[t]
\vspace{-6ex}
\begin{center}
\includegraphics[scale=.55]{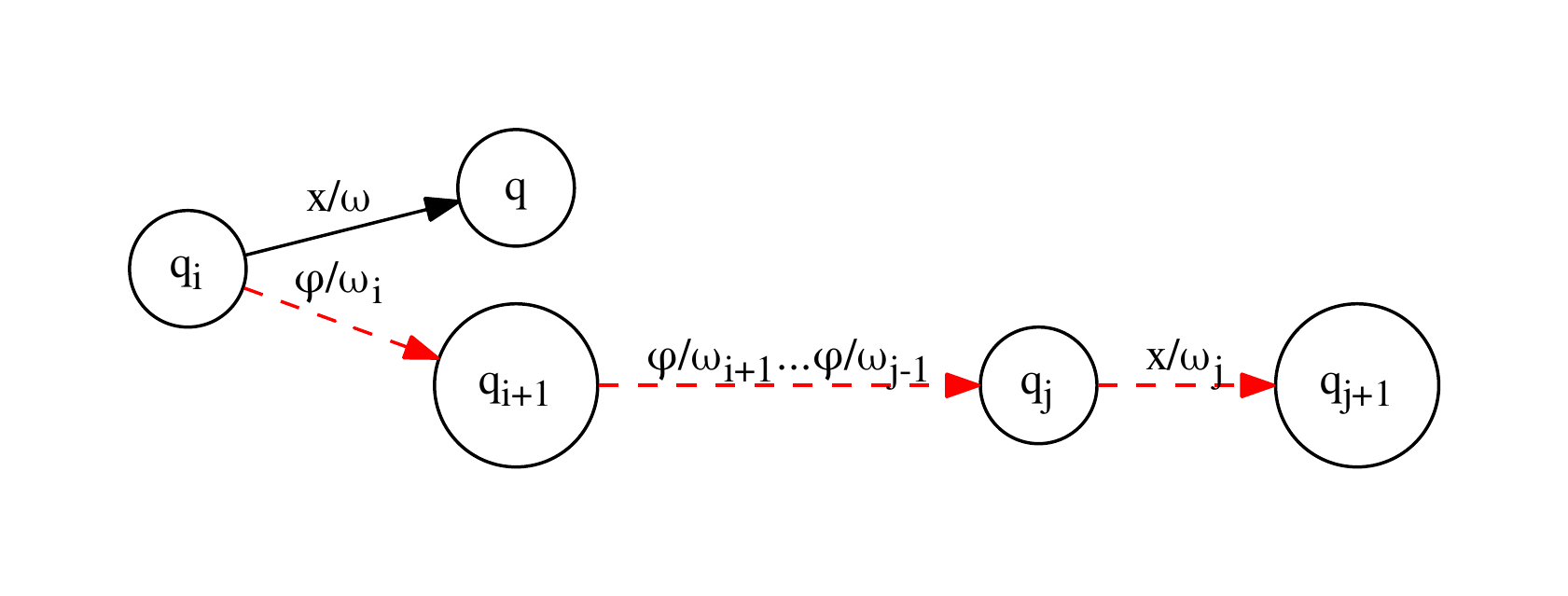}\\
\end{center}
\vspace{-6ex}
\caption{The (dashed red) path $e_i = (q_i, \varphi, \omega_i, q_{i+1})$ to
  $e_j = (q_j, x, \omega_j, q_{j+1})$ is disallowed since $x$ can
  be read already on $e = (q_i, x, \omega, q)$.
\label{fig:forbid}}
\end{figure}

We use $P^*(q, q') \subseteq P(q, q')$ to denote
the set of (not dis-) allowed paths from state $q$ to $q'$ in a $\varphi$-WFA.
This again extends to sets in the obvious way.
A path $\pi$ is successful in
a $\varphi$-WFA if $\pi \in P^*(i, f)$ and only in that
case is the input string $\alpha = \lab{\pi}$ accepted.

The language accepted by the $\varphi$-automaton $A$ is the set
$\cL(A) = \{\alpha \in \Sigma^* : \alpha=\lab{\pi}, \pi \in P^*(i,f) \}$.
The weight of  $\alpha \in \Sigma^*$ assigned by the automaton is
$A{(\alpha)} = \Sigma_{\pi \in P^*(i, f) : \, \lab{\pi} = \alpha}
\weight{\pi}$. We assume each string in $\cL(A)$ is terminated by the
symbol $\eos$ as before. We also assume there are no $\varphi$-labeled cycles and
there is at most one exiting failure transition per state.

We express the
{\em $\varphi$-extended transitions} leaving $q$ as
\begin{align*}\
  E^*[q] = \Big \{ (q,x,\omega,q')
:\quad & \pi \in P^*(q,Q),
   x = \lab{\pi} = \lastlab{\pi} \in \Sigma,
 \omega = \weight{\pi}, q' = \dst{\pi} \Big \}.
\end{align*}
This is a set of (possibly new) transitions $(q,x,\omega,q')$, one for each
allowed path from previous state $q$ to next state $q'$
with optional leading failure transitions and a final $x$-labeled transition.
Denote the labels of $E^*[q]$ by $L^*[q]$. Note
the WFA $(\Sigma, Q, E^*, i, f)$ accepts the same strings with the same weights
as $\varphi$-WFA $A$ and thus $\cL(A)$ is regular \cite{allauzen2018}.

A {\em probabilistic (or stochastic) $\varphi$-WFA} satisfies
\[
\sum_{e \in E^*[q]} \weight{e} = 1 \text{ and } \weight{e} \geq 0,
\quad\forall q \in Q - \{f\}.
\]

A deterministic $\varphi$-WFA is {\em backoff-complete} if a
failure transition from state $q$ to $q'$ implies $L[q] \cap \Sigma
\subseteq L[q'] \cap \Sigma$. Further, if $\varphi \notin L[q']$,
then the containment is strict: $L[q] \cap \Sigma \subset L[q'] \cap \Sigma$.
In other words, if
a symbol can be read immediately from a state $q$ it can also be read
from a state failing ({\em backing-off}) from $q$ and if $q'$ does not
have a backoff arc, then at least one additional label can be read
from $q'$ that cannot be read from $q$. For example, both topologies
depicted in Figure~\ref{fig:top} have this property.  We restrict
our target automata to have a topology
with the backoff-complete property since it will simplify our analysis,
make our algorithms efficient and is commonly found in applications.
For example backoff $n$-gram models, such
the Katz model, are $\varphi$-cycle-free, backoff-complete $\varphi$-WFAs \cite{katz,Chen98}.

For a symbol $x \in \Sigma$ and a state $q \in Q$
of a deterministic, probabilistic $\varphi$-WFA A, define
$p^*_a(x | q) \triangleq w \text { if } (q, x, w, q') \in E^*[q]$ and $p^*_a(x | q) \triangleq 0$
otherwise. Then $p^*_a$ is a probabilistic model over $\Sigma$ as defined
in Section~\ref{subsec:probmod}. Note the distribution
$p^*_a$ at a state $q$ is defined over the $\varphi-$extended transitions
$E^*[q]$ where $p_a$ in the previous section is defined
over the transitions $E[q]$. It is convenient to define
a companion distribution $p_a \in \sP(A)$
to $p_a^*$ as follows:\footnote{The meaning of $\sP(A)$
when $A$ is $\varphi$-WFA is to interpret it as a WFA with
the failure labels as regular symbols.}
given a symbol $x \in \Sigma\cup\{\varphi\}$ and state
$q \in Q$, define $p_a(x | q) \triangleq p^*_a(x | q)$
when $x \in L[q]\cap\Sigma$,
$p_a(\varphi | q) \triangleq
1 - \sum_{x \in L[q]\cap\Sigma} p^*_a(x | q)$, and $p_a(x | q) \triangleq 0$ otherwise. The companion distribution is thus defined solely over the transitions
$E[q]$.

When $A = (\Sigma, Q, E, i, f)$ is an unweighted
deterministic, backoff-complete $\varphi$-WFA,
we denote by $\sP^*(A)$ the set of all probabilistic models $p^*_a$
representable as a weighted
$\varphi$-WFA $\weighted{A} = (\Sigma, Q, \weighted{E}, i, f)$
of same topology as $A$ with
\begin{align*}
\weighted{E}   = & \{(q,x,p_a(x|q),q') : (q,x,1,q') \in E, x \in \Sigma \}\, \cup \\
& \{(q,\varphi,\alpha(q,q'),q') : (q,\varphi,1,q') \in E\},
\end{align*}
where $p_a \in \sP(A)$ is the companion distribution to $p_a^*$ and
$\alpha(q,q') = p_a(\varphi|q)/d(q,q')$
is the weight of the failure transition from state $q$ to $q'$ with
\begin{equation}
d(q,q') = 1 - \sum_{x \in L[q]\cap\Sigma} p_a(x|q').
\label{eq:denom}
\end{equation}
Note we have specified the weights on the
automaton that represents $p_a^* \in \sP^*(A)$ entirely in terms
of the companion distribution $p_a \in \sP(A)$.
The failure transition weight $\alpha$ is determined from
the requirement that $\sum_x E^*[q] = 1$ \cite[Equation 16]{katz}.
Thanks to the backoff-complete property, this transition weight
depends only on its previous state $q$ and its next state
$q'$ and not all the states in the failure path from $q$.
This is a key reason for assuming that property; otherwise our algorithms could
not benefit from this locality.

Conversely, each distribution $p_a \in \sP(A)$ can be associated
to a distribution $p^*_a \in \sP^*(A)$
given a deterministic, backoff-complete $\varphi$-WFA $A$.
First extend $\alpha(q,q')$ to any failure path as follows.
Denote a failure path from state $q$ to
$q'$ by $\pi_\varphi(q,q')$. Then define
\begin{equation}
  \label{eq:backoff}
\alpha(q,q') = \prod_{e \in \pi_\varphi(q,q')}
\frac{p_a(\varphi|\src{e})}{d(\src{e}, \dst{e})},
\end{equation}
where this quantity is taken to be $1$ when the failure path is empty
($q = q'$). Finally define
\begin{equation}
p^*_a(x|q) =
\begin{cases}
\alpha(q,q^x)\, p_a(x|q^x),&x \in L^*[q]\\
0, &\text{otherwise}
\end{cases}
\label{eq:pstar}
\end{equation}
where for $x \in L^*[q]$, $q^x$ signifies the first state $q'$ on a
$\varphi$-labeled path in $A$ from state $q$ for which $x \in L[q']$.

For \eqref{eq:backoff} to be well-defined, we need $d(\src{e},
\dst{e}) > 0$. To ensure this condition, we restrict $\sP(A)$ to
contain distributions such that $p_a(x|q) \geq \epsilon$ for each $x
\in L[q]$.\footnote{For brevity, we do not include $\epsilon$ in the
  notation of $\sP(A)$.}

Given an unweighted deterministic, backoff-complete,
automaton $A$, our goal is to find the target distribution $p^*_a \in
\sP^*(A)$ that has the minimum KL divergence from our source
probability model $p_s$.

\begin{lemma}
Assume $\cL(p_s) \subseteq \cL(A)$. Let $p^* = \best{p}(x|q_a)$ from
Lemma~\ref{thm:exact}.
If $\cL(p^*) \subseteq \cL(A)$ and $p^* \in \sP^*(A)$
then
\[
p^* = \argmin_{p^*_a \in \sP^*(A)} D(p_s || p^*_a).
\]
\label{thm:phiexact}
\end{lemma}
\begin{proof}
This follows immediately from Lemma~\ref{thm:exact}.
\end{proof}

The requirement that the $p^*$ of Lemma~\ref{thm:phiexact} is in
$\sP^*(A)$ will be true if, for instance, the target has no failure
transitions or if the source and target are both $\varphi$-WFAs with
the same topology and failure transitions. In general, this
requirement can not be assured. While membership in $\sP(A)$
principally requires the weights of the transitions leaving a state
are non-negative and sum to $1$, membership in $\sP^*(A)$
imposes additional constraints due to the failure transitions,
indicated in Equation~\ref{eq:pstar}.

As such, we restate our
goal in terms of the companion distribution $p_a \in \sP(A)$
rather than its corresponding distribution $p^*_a \in \sP^*(A)$ directly.
Let $B_n(q)$ be the set of states in
$A$ that back-off to state $q$ in $n$ failure transitions and let
$B(q) = \bigcup^{|Q_a|}_{n=0} B_n(q)$.

\begin{lemma}
If $\cL(p_s) \subseteq \cL(A)$ then
\[
\argmin_{p^*_a \in \sP^*(A)} D(p_s || p^*_a)
 = \argmax_{p_a \in \sP(A)}
\sum_{q \in Q_a} \bigg \{
\sum_{x \in L[q]} C(x,q) \log p_a(x|q)
- \sum_{q_0 \in B_1(q)} C(\varphi, q_0) \log d(q_0, q)
\bigg \},
\]
where
\begin{align}
C(x, q) & = \sum_{q_a \in B(q)} c(x, q_a)\indic_{q = q_a^x},\quad x \in \Sigma, \label{eq:cfactor1}  \\
C(\varphi, q) & =  \sum_{q_a \in B(q)} \sum_{x \in \Sigma} c(x, q_a) \indic_{x \notin L[q]}. \label{eq:cfactor2}
\end{align}
and do not depend on $p_a$.
\label{thm:phigen}
\end{lemma}
\begin{proof}
  From Corollary~\ref{thm:kldivmin}, Equation~\ref{eq:pstar} and the previously
  shown 1:1 correspondence between each distribution $p^*_a \in \sP^*(A)$
  and its companion distribution $p_a \in \sP(A)$
\begin{align}
\argmin_{p^*_a \in \sP^*(A)} D(p_s || p^*_a) & =
\argmin_{p^*_a \in \sP^*(A)} \sum_{q_a \in Q_a}
    \sum_{x\in L^*[q_a]} c(x, q_a) \log p^*_a(x|q_a) \nonumber\\
& = \argmax_{p_a \in \sP(A)} \sum_{q_a \in Q_a}
    \sum_{x\in L^*[q_a]} c(x, q_a) \log
    \alpha(q_a,q^x_a) p_a(x|q^x_a)\nonumber\\
& = \argmax_{p_a \in \sP(A)} \sum_{q_a \in Q_a} \sum_{x \in L^*[q_a]}
c(x, q_a) \log \prod_{e \in \pi_\varphi(q_a,q^x_a)}
\frac{p_a(\varphi|\src{e})}{d(\src{e}, \dst{e})} p_a(x|q^x_a) \nonumber\\
& = \argmax_{p_a \in \sP(A)} \bigg \{ A_x + A_\varphi - A_d \bigg \}, \label{eq:maxfactors}
\end{align}
where we distribute the factors inside the logarithm in
Equation~\ref{eq:maxfactors} as follows:
\begin{align}
A_x & = \sum_{q_a \in Q_a} \sum_{x \in L^*[q_a]} c(x, q_a) \log p_a(x|q^x_a) \nonumber \\
& = \sum_{q \in Q_a} \sum_{q_a \in B(q)} \sum_{x \in L[q]\cap\Sigma}
c(x, q_a) \indic_{q=q^x_a} \log p_a(x|q) \label{eq:indicx}\\
& = \sum_{q \in Q_a} \sum_{x\in L[q]\cap\Sigma} C(x, q) \log p_a(x|q).\nonumber
\end{align}
Equation~\ref{eq:indicx} follows from $q = q^x_a$ implying
$q_a \in B(q)$.
\begin{align}
A_\varphi & = \sum_{q_a \in Q_a } \sum_{x \in L^*[q_a]}
c(x, q_a) \log
\prod_{e \in \pi_\varphi(q_a,q^x_a)} p_a(\varphi|\src{e})\nonumber\\
& =  \sum_{q_a \in Q_a} \sum_{x \in L^*[q_a]}
c(x, q_a)
\sum_{e \in \pi_\varphi(q_a,q^x_a)} \log
p_a(\varphi|\src{e})\nonumber\\
& =  \sum_{q \in Q_a}\sum_{q_a \in B(q)} \sum_{x \in L^*[q_a]}
c(x, q_a)
\sum_{e \in \pi_\varphi(q_a,q^x_a)} \indic_{q=p[e]}
\log
p_a(\varphi|q)\nonumber\\
& = \sum_{q \in Q_a} \sum_{q_a \in B(q)} \sum_{x \in \Sigma}
c(x, q_a) \indic_{x \notin L[q]} \log p_a(\varphi|q) \label{eq:indicphi}\\
& = \sum_{q \in Q_a} C(\varphi, q) \log p_a(\varphi|q).\nonumber
\end{align}
Equation~\ref{eq:indicphi} follows from $e \in \pi_\varphi(q_a, q^x_a)$
implying $x \notin{\src{e}}$.
\begin{align}
A_d & = \sum_{q_a \in Q_a} \sum_{x \in L^*[q_a]}
c(x, q_a) \log
\prod_{e \in \pi_\varphi(q_a,q^x_a)} d(\src{e}, \dst{e})\nonumber\\
& = \sum_{q_a \in Q_a} \sum_{x \in L^*[q_a]}
c(x, q_a)
\sum_{e \in \pi_\varphi(q_a,q^x_a)} \log
 d(\src{e}, \dst{e})\nonumber\\
& = \sum_{q \in Q_a} \sum_{q_a \in B(q_0)} \sum_{q_0 \in B_1(q)}  \sum_{x \in L^*[q_a]}
c(x, q_a)
\sum_{e \in \pi_\varphi(q_a,q^x_a)} \indic_{q_0=\src{e}}\log
 d(q_0, q)\nonumber\\
& = \sum_{q \in Q_a} \sum_{q_a \in B(q_0)} \sum_{q_0 \in B_1(q)} \sum_{x \in \Sigma}
c(x, q_a) \indic_{x \notin L[q_0]} \log d(q_0, q) \nonumber \\
& =  \sum_{q \in Q_a} \sum_{q_0\in B_1(q)} C(\varphi, q_0) \log d(q_0, q).\nonumber
\end{align}
Substituting these results into Equation~\ref{eq:maxfactors} proves the lemma.
\end{proof}

If there are no failure transitions in the target
automaton, then $C(x, q) = c(x, q)$,
if $x$ is in $\Sigma$, and is $0$ otherwise.
In this case, the statement of Lemma~\ref{thm:phigen} simplifies to that
of Corollary~\ref{thm:kldivmin}.

Unfortunately, we do not have a closed-form solution, analogous
to Lemma~\ref{thm:exact} or Lemma~\ref{thm:phiexact}, for the general
$\varphi$-WFA case. Instead we will present
numerical optimization algorithms to find the KL divergence minimum
in the next section. This is aided by the observation that
the quantity in braces in the statement of
Lemma~\ref{thm:phigen} depends
on the distribution $p_a$ only at state $q$. Thus the KL divergence minimum
can be found by maximizing that quantity independently for each state.

%% file: approx_algo.tex
Approximating a probabilistic source algorithmically
as a weighted finite automaton requires
two steps: (1) compute the quantity $c(x, q_a)$ found in 
Corollary~\ref{thm:kldivmin} and Lemma~\ref{thm:exact} or $C(x, q)$ in
Lemma~\ref{thm:phigen} and (2) use this quantity
to find the minimum KL divergence solution. The first step, which
we will refer to as {\em counting}, is covered in the next section and 
the KL divergence minimization step is covered afterwards, followed
by an explicit comparison of the presented algorithms with closely
related prior work.

\subsection{Counting}
How the counts are computed will depend on the form of the source
and target models. We break this down into several cases.

\subsubsection{WFA source and target}
\label{sec:wfa_source}
When the source and target models are represented as WFAs
we compute $c(x, q_a)$ from Lemma~\ref{thm:exact}.
From  Equation~\ref{eq:cfactor0} this can be written as
\begin{equation}
c(x, q_a) = \sum_{q_s \in Q_s}  \gamma(q_s, q_a) p_s(x|q_s),
\label{eq:transcnt1}
\end{equation}
where
\[
\gamma(q_s, q_a) = \sum^\infty_{i=0} 
\sum_{x^i: q_s(x^i)=q_s, q_a(x^i)=q_a} p_s(x^i).
\]

The quantity $\gamma(q_s, q_a)$ can be computed as
\[
\gamma(q_s, q_a) =
 \sum_{\pi \in P_{S\cap A}((i_s,i_a), (q_s, q_a))} \weight{\pi},
\label{eq:statecnt}
\]
where $S \cap A$ is the weighted finite-state intersection of automata
$S$ and $A$ \cite{Mohri2009}. The above summation over this intersection
is the (generalized)
{\em shortest distance} from the initial state to a specified state computed
over the {\em positive real semiring} \cite{mohri2002,allauzen2018}. 
Algorithms to efficiently
compute the intersection and shortest distance on WFAs are available in
{\tt OpenFst} \cite{allauzen2007}, 
an open-source weighted finite automata library. 

Then from Equation~\ref{eq:transcnt1} we can form the sum
\begin{equation}
c(x, q_a) = \sum_{((q_s, q_a), x, w, (q_s', q_a')) \in E_{S \cap A}}  
\gamma(q_s, q_a)\, w.
\label{eq:transcnt2}
\end{equation}
Equation~\ref{eq:transcnt2} is
the weighted count of the paths in $S\cap A$ that
begin at the initial state and end
in any transition leaving a state $(q_s, q_a)$ labeled with $x$.

The worst-case time complexity for the counting step
is dominated by the shortest distance algorithm
on the intersection $S \cap A$. The shortest distance computation is
a meta-algorithm that depends on the queue discipline selected 
\cite{mohri2002}. If $s$ is the maximum number of times a state in
the intersection
is inserted into the shortest distance queue, $C$ the maximum cost of
a queue operation, and $D$ the maximum out-degree in $S$ and $A$ 
(both assumed deterministic), then the algorithm runs in
$O(s(D+C)|Q_S||Q_A|)$ \cite{mohri2002,allauzen2018}.
The worst-case space complexity is
in $O(D |Q_S||Q_A|)$, determined by the intersection size.
\subsubsection{$\varphi$-WFA source and target}
\label{sec:exact_counting}
When the source and target models
are represented as $\varphi$-WFAs we compute
$C(x, q_a)$ from Lemma~\ref{thm:phigen}.
From Equation~\ref{eq:cfactor1} and the previous case
this can be written as
\begin{equation}
C(x, q)  = \sum_{q_a \in B(q)} \sum_{q_s \in Q_s} \gamma(q_s, q_a) 
p_s(x|q_s) \indic_{q = q_a^x},\quad x \in \Sigma.
\label{eq:phitranscnt1}
\end{equation}
To compute this quantity we first form $S \cap A$
using an efficient $\varphi$-WFA intersection 
that compactly retains failure transitions in the result 
as described in \cite{allauzen2018}.
Equation~\ref{eq:phitranscnt1} is 
the weighted count of the paths in $S\cap A$
allowed by the failure transitions that begin at the initial state 
and end in any transition leaving a state $(q_s, q)$ labeled with $x$.

We can simplify this computation by the following transformation.
First we convert $S\cap A$ to an equivalent WFA by 
replacing each failure transition
with an epsilon transition and introducing a negatively-weighted transition
to compensate for formerly disallowed paths \cite{allauzen2018}.
\begin{figure}[t]
\vspace{-6ex}\hspace*{-6ex}
\begin{tabular}{c@{}c}
\includegraphics[scale=.6]{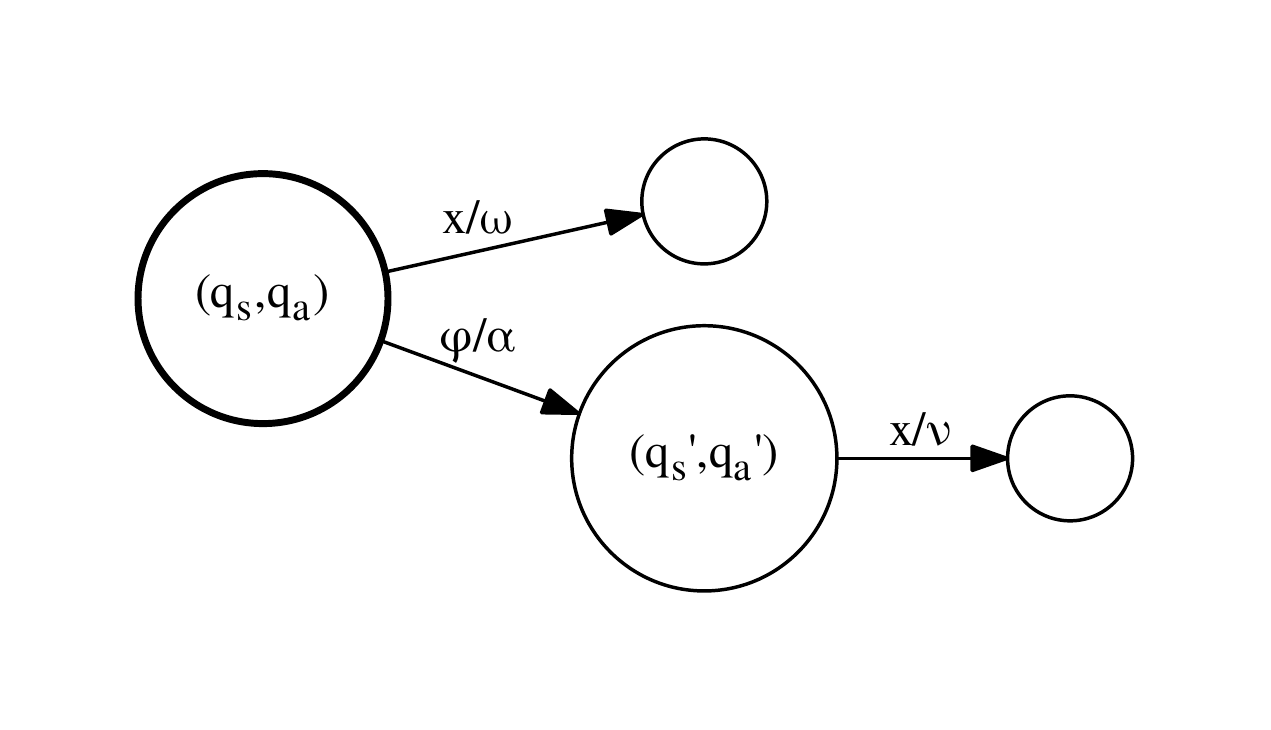} \hspace*{-8ex}& 
\includegraphics[scale=.6]{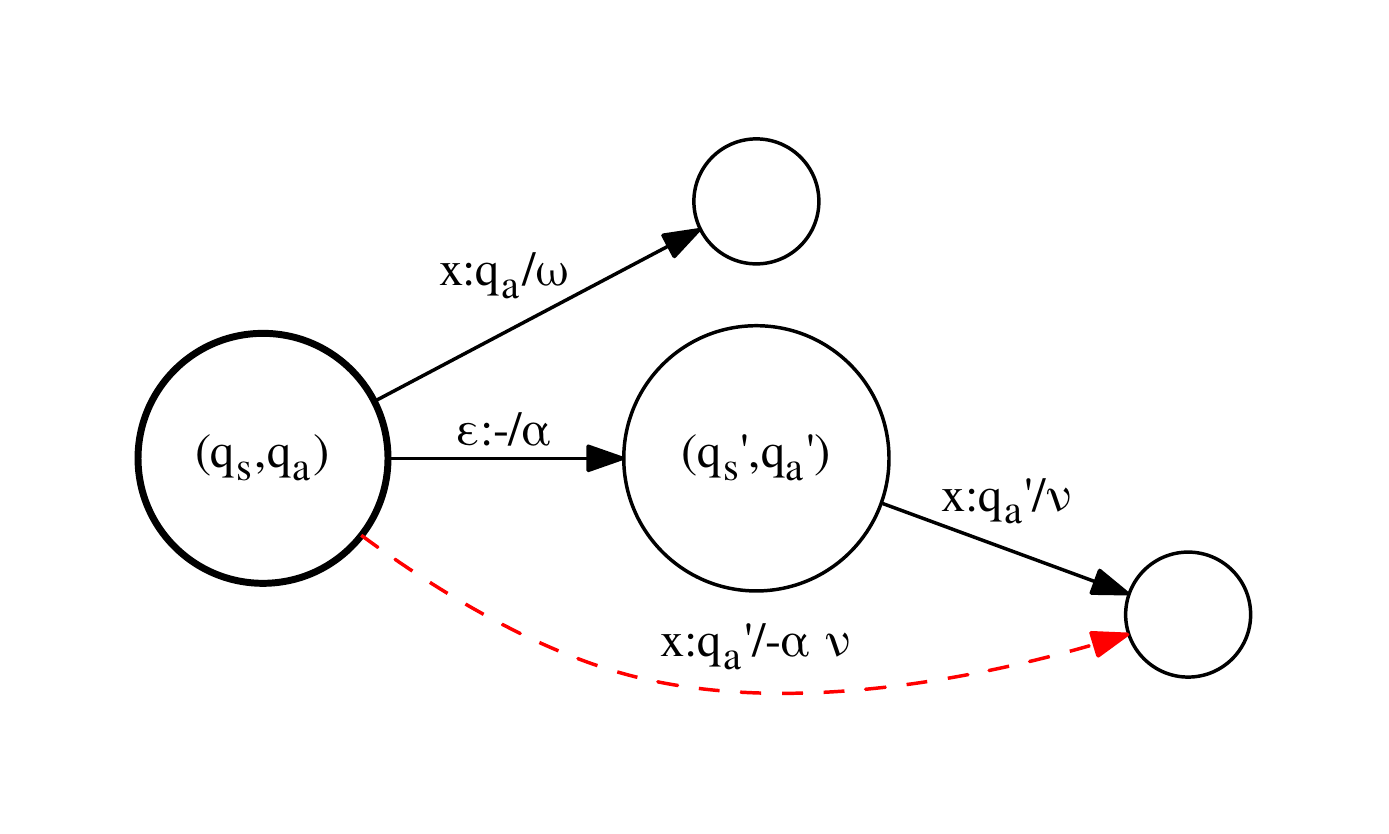}\\
$S \cap A$ & $T$
\end{tabular}
\caption{A $\varphi$-WFA is transformed into an equivalent WFA
by replacing each failure transition by an $\epsilon$-transition. 
To compensate for the formerly disallowed paths, new (dashed
red) negatively-weighted transitions are added.
The result is promoted to a transducer $T$ with the output label
used to keep track of the previous state in $A$ of the compensated 
positive transition.}
\label{fig:phitranscnt}
\end{figure}
The result is then promoted to a transducer $T$ with the
output label used to keep track of the previous state
in $A$ of the compensated 
positive transition (see Figure~\ref{fig:phitranscnt}).\footnote{The 
construction illustrated in Figure~\ref{fig:phitranscnt}
is sufficient when $S \cap A$
is acyclic. In the cyclic case a slightly modified construction is needed
to ensure convergence in the shortest
distance calculation \cite{allauzen2018}}.
Algorithms to efficiently
compute the intersection and shortest distance on $\varphi$-WFAs
are available in the {\tt OpenGrm} libraries \cite{allauzen2018}.

Then
\begin{equation}
C(x, q) =  
\sum_{((q_s, q_a), x, q, w, (q_s', q_a')) \in E_T} 
\gamma_T(q_s, q_a) w, \quad x \in \Sigma,
\label{eq:phitranscnt2}
\end{equation}
where $e = (\src{e}, il[e], ol[e], \weight{e}, \dst{e})$ is a
transition in $T$ and $\gamma_T(q_s, q)$ is the shortest distance from the initial
state to $(q_s, q_a)$ in $T$ computed over the {\em real semiring}
as described in \cite{allauzen2018}.
Equation~\ref{eq:phitranscnt2} is
the weighted count of all paths in $S \cap A$
that begin at the initial state 
and end in any transition leaving a state $(q_s, q)$ labeled with $x$
minus the weighted count of those paths that are disallowed
by the failure transitions.

Finally, we compute $C(\varphi, q)$ as follows. The count mass entering a state
$q$ must equal the count mass leaving a state
\begin{align*}
\sum_{(q_a, x, 1, q) \in E_A} C(x, q_a) &= \sum_{(q, x', 1, q_a') \in E_A} C(x',
 q)\\
&=\sum_{\substack{
(q, x', 1, q_a') \in E_A, x' \in \Sigma}} C(x', q) +  C(\varphi, q)
\end{align*}
Thus
\[
C(\varphi, q) = 
\sum_{(q_a, x, 1, q) \in E_A} C(x, q_a) - \sum_{\substack{
(q, x', 1, q_a') \in E_A, x' \in \Sigma}} C(x', q).
\]
This quantity can be computed iteratively in the topological order of states
with respect to the $\varphi$-labeled transitions.

The worst-case 
time and space complexity for the counting step for $\varphi$-WFAs
is the same as for WFAs \cite{allauzen2018}. 

\subsubsection{Arbitrary source and $\varphi$-WFA target}
\label{sec:arbitrary_counting}

In some cases, the source is a distribution with possibly infinite
states, e.g., LSTMs. For these sources, computing $C(x,q)$ can be
computationally intractable as \eqref{eq:phitranscnt1} requires a
summation over all possible states in the source machine, $Q_s$.
We propose to use a sampling approach to approximate $C(x,q)$ for
these cases.  Let $x(1), x(2),\ldots, x(N)$ be independent random
samples from $p_s$. Instead of $C(x,q)$, we propose to use
\begin{equation*}
\hat{C}(x, q)  = \sum_{q_a \in B(q)} \sum_{q_s \in Q_s} \hat{\gamma}(q_s, q_a) 
p_s(x|q_s) \indic_{q = q_a^x},\quad x \in \Sigma,
\end{equation*}
where 
\begin{equation*}
\hat{\gamma}(q_s, q_a)  =
\frac{1}{N} \sum^N_{j=1} \sum_{i\geq 1}  \indic_{q_s(x^i(j)) = q_s, q_a(x^i(j)) = q_a}.
\end{equation*}
Observe that in expectation,
\begin{align*}
\EE[\hat{\gamma}(q_s, q_a)] &=
\frac{1}{N} \sum^N_{j=1} \sum_{i\geq 1}  \EE[\indic_{q_s(x^i(j)) = q_s, q_a(x^i(j)) = q_a}] \\
&= \sum_{i\geq 1} \sum_{x^i : q_s(x^i) = q_s, q_a(x^i) = q_a} p_s(x^i ) \\
&= \gamma(q_s, q_a),
\end{align*}
and hence $\hat{\gamma}(q_s, q_a)$ is an unbiased, asymptotically consistent
estimator of
$\gamma(q_s, q_a)$. Given $\hat{C}(x,q)$, we compute $C(\varphi, q)$
similarly to the previous section.
If $\ell$ is the expected number of symbols per sample, then the
computational complexity of counting in expectation is in $O(N \ell |\Sigma|)$.

\subsection{KL divergence minimization}
\label{sec:kl}
\subsubsection{WFA target}
When the target topology is a deterministic WFA, we use $c(x, q_a)$ from the
previous section and Lemma~\ref{thm:exact} to immediately find the
minimum KL divergence solution.

\subsubsection{$\varphi$-WFA target}
When the target topology is a deterministic, backoff-complete 
$\varphi$-WFA, Lemma~\ref{thm:phiexact}
can be applied in some circumstances to find the minimum KL divergence
solution but not in general.  However, as noted before,
the quantity in braces in the statement of
Lemma~\ref{thm:phigen} depends on the distribution $p_a$ only at state
$q$ so the minimum KL divergence $D(p_s||p^*_a)$ can be found by
maximizing that quantity independently for each state.

Fix a state $q$ and let  $y_x \triangleq p_a(x|q)$
for $x \in L[q]$ and let $\yvec \triangleq 
[y_x]_{x \in L[q]}$\footnote{We fix some total order on 
$\Sigma\cup\{\varphi\}$ so that
$\yvec$ is well-defined.}. Then our goal reduces to
\begin{equation}
\argmax_{\yvec}  \sum_{x \in L[q]} C(x,q) \log y_x
- \sum_{q_0 \in B_1(q)} C(\varphi, q_0) \log 
\big ( 1 - \sum_{x \in L[q_0]\cap\Sigma} y_x \big ).
\label{eq:klopt}
\end{equation}
subject to the constraints $y_x \geq \epsilon$ for $x \in L[q]$ and
$\sum_{x\in L[q]} y_x = 1$.

This is a difference of two concave functions in $\yvec$ since
$\log(f(\yvec))$ is concave for any linear function $f(\yvec)$,
$C(x, q), C(\varphi, q_0)$ are always non-negative and the sum of concave functions is
also concave. We give a {\em DC programming} solution to this
optimization \cite{horst1999}. Let
\[
\Omega = \{ \yvec : \forall x, y_x \geq \epsilon, \sum_{x \in L(q)} y_x \leq 1 \},
\]
and let $u(\yvec) = \sum_{x \in L[q]} C(x,q) \log y_x $ and $v(\yvec)
= \sum_{q_0 \in B_1(q)} C(\varphi, q_0) \log
\big ( 1 - \sum_{x \in L[q_0]\cap\Sigma} y_x \big )$.
Then the optimization problem can be written as
\[
\max_{\yvec \in \Omega} u(\yvec) - v(\yvec).
\]
The {\em DC programming solution} for such a problem uses an iterative
procedure that linearizes the subtrahend in the concave difference
about the current estimate and then solves the resulting concave
objective for the next estimate
\cite{horst1999} i.e.,
\[
\yvec^{n+1} = \argmax_{\yvec \in \Omega} u(\yvec) - \yvec \cdot \triangledown v(\yvec^{n}).
\]
Substituting $u$ and $\triangledown v$ gives
\begin{align}
\yvec^{n+1} 
 = & \argmax_{\yvec \in \Omega} 
\sum_{x \in L[q]} \bigg \{
C(x,q) \log  y_x + 
y_x f(x, q, \yvec^n) \bigg \},
\label{eq:objective_y}
\end{align}
where
\begin{equation}
f(x, q, \yvec^n) = \sum_{q_0 \in B_1(q)} 
\frac {  C(\varphi, q_0) \indic_{x \in L[q_0]\cap\Sigma}}
{1 - \sum_{x' \in L[q_0]\cap\Sigma} y_{x'}^n}.
\label{eq:dcfactor}
\end{equation}
Observe that ${1 - \sum_{x' \in L[q_0]\cap\Sigma} y_{x'}^n}
\geq \epsilon$ as the automaton is backoff-complete and $\yvec^n \in \Omega$.

Let $C(q)$ be defined as:
\[
C(q) = \sum_{x' \in L[q]} C(x',q).
\]

The following lemma provides the solution to the optimization problem
in \eqref{eq:objective_y} which leads to a stationary point of the objective.

\begin{lemma}
Solution to \eqref{eq:objective_y} is given by
\begin{equation}
\label{eq:iteration_y}
y^{n+1}_x = \max \left(\frac{C(x,q)}{\lambda - f(x,q,\yvec^n)}, \epsilon \right),
\end{equation}
where $\lambda \in \left[ \max_{x\in L[q]} f(x, q, \yvec^n) + C(x,
q), \max_{x\in L[q]} f(x, q, \yvec^n) + \frac{C(q)}{1-|L[q]|\epsilon} \right]$ such that
$\sum_{x} y^{n+1}_x = 1$.
\end{lemma}
\begin{proof}
With KKT multipliers, the optimization problem can be written as
\[
\max_{\yvec, \lambda, \mu_x : \mu_x \leq 0}
\sum_{x \in L[q]} \bigg \{
C(x,q) \log y_x + y_x f(x, q, \yvec^n) \bigg \} + \lambda \bigl(1 - \sum_{x \in L[q]}
y_x \bigr) + \sum_{x \in L[q]} \mu_x ( \epsilon - y_x).
\]
We divide the proof into two cases depending on the value of
$C(x,q)$. Let $C(x,q) \neq 0$.  Differentiating the above equation with respect to $y_x$ and equating to zero, we
get
\[
y^{n+1}_x = \frac{C(x,q)}{\lambda +\mu_x - f(x, q, \yvec^n)}.
\]
Furthermore, by the KKT condition, $\mu_x (\epsilon - y^{n+1}_x) = 0$.
Hence, $\mu_x$ is only non-zero if $y^{n+1}_x
= \epsilon$ and if $\mu_x$ is zero, then $y^{n+1}_x
= \frac{C(x,q)}{\lambda  - f(x, q, \yvec^n)}$.
Furthermore, since for all $x$, $\mu_x \leq 0$, for
$y^{n+1}_x$ to be positive, we need $\lambda \geq \max_x f(x,
q, \yvec^n)$. Hence, the above two conditions can be re-expressed
as \eqref{eq:iteration_y}.  If $C(x,q) = 0$, then we get
\[
f(x, q, \yvec^n) =  \lambda + \mu_x \text{ and } \mu_x(\epsilon - y^{n+1}_x) = 0,
\]
and the solution is given by $y^{n+1}_x = \epsilon$ and $\mu_x =
f(x,q, \yvec^n) - \lambda$. Since $\mu_x$ cannot be positive, we have
$f(x,q, \yvec^n) \leq \lambda$ for all $x$. Hence, irrespective of the value of
$C(x,q)$, the solution is given by \eqref{eq:iteration_y}.

The above analysis restricts $\lambda \geq \max_{x} f(x,
q,\yvec^n)$. If $\lambda < f(x, q, \yvec^n) + C(x, q)$, then $y^{n+1}_x >
1$ and if $\lambda > \max_{x} f(x, q, \yvec^n) + \frac{C(q)}{1-|L[q]|\epsilon}$, then $\sum_x y^{n+1}_x <
1$. Hence $\lambda$ needs to lie in
\[
\left[ \max_{x\in L[q]} f(x, q, \yvec^n) + C(x, q),   \max_{x\in L[q]} f(x, q, \yvec^n) + \frac{C(q)}{1-|L[q]|\epsilon} \right]
\]
to ensure that $\sum_{x} y^{n+1}_x = 1$.
\end{proof}
From this, we form the \textsc{KL-Minimization} algorithm in Figure \ref{fig:klminalg}. Observe that if
all the counts are zero, then for any $\yvec$, $u(\yvec) - v(\yvec) =
0$ and any solution is an optimal solution and the algorithm returns a
uniform distribution over labels. In other cases, we initialize the
model based on counts such that $\yvec^0 \in \Omega$. We then repeat the
DC programming algorithm iteratively until convergence. Since,
$\Omega$ is a convex compact set and functions $u$, $v$, and
$\triangledown v$ are continuous and differentiable in $\Omega$,
the \textsc{KL-Minimization} converges to a stationary point \cite[Theorem $4$]{sriperumbudur2009convergence}.
For each state $q$, the computational complexity of \textsc{KL-Minimization} is in $O(|E[q]|)$
per iteration. Hence, if the maximum number of iterations per each state is $s$,
the overall computational complexity of \textsc{KL-Minimization} is in $O(s |E|)$.

\begin{figure*}[t]
\begin{center}
\fbox{\begin{minipage}{1.0\textwidth}
\begin{center}
Algorithm \textsc{KL-Minimization}
\end{center}
\noindent\textbf{Notation:}\\
\begin{tabular}{ll@{\hspace*{0.5in}}ll}
$\bullet$ &  $y_x = p_a(x| q)$ for $x \in L(q)$ & $\bullet$ &  $k = |L[q]|$
\\
$\bullet$ &  $C(x,q)$ from Equations~\ref{eq:cfactor1} and \ref{eq:cfactor2} & $\bullet$ &  $\textrm{lb} =  \max_{x\in L[q]} f(x, q, \yvec^n) + C(x, q)$
\\
$\bullet$ &  $C(q) = \sum_{x' \in L[q]} C(x',q)$ & $\bullet$ &  $\textrm{ub} =  \max_{x\in L[q]} f(x, q, \yvec^n) + \frac{C(q)}{1-k\epsilon}$
\\
$\bullet$ &  $f(x,q, \yvec^n)$ from Equation~\ref{eq:dcfactor} &
$\bullet$ &  $\epsilon = $ lower bound on $y_x$\\~
\vspace{-0.2cm}
\end{tabular}

\noindent\textbf{Trivial case}: If $C(q)= 0$, output $\yvec$ given by $y_x = 1/k$ for all $x$.

\noindent\textbf{Initialization}: 
Initialize:
\[
y^0_x = \frac{C(x,q)}{C(q)} \left( 1 - k \epsilon \right) + \epsilon.
\]

\noindent\textbf{Iteration}: Until convergence do:
\[
 y^{n+1}_x = \max \left( \frac{C(x,q)}{\lambda - f(x,q,\yvec^n)}, \epsilon \right),
\]
where $\lambda \in \left[\textrm{lb},\textrm{ub} \right]$ is chosen
(in a binary search) to ensure $\sum_{x \in L(q)} y_x = 1$.
\end{minipage}}
\end{center}
\caption{\textsc{KL-Minimization} Algorithm.}\label{fig:klminalg}
\end{figure*}

\subsection{Discussion}
\label{sec:prev_work_discussion}
Our method for approximating from source $\varphi$-WFAs, specifically from
backoff $k$-gram models, shares some similarities
with entropy pruning \cite{stolcke2000entropy}. Both can be used to approximate
onto a more compact $k$-gram topology and both use
a KL-divergence criterion to do so. They differ in that entropy pruning, 
by sequentially removing transitions, selects the target topology and in doing
so uses a greedy rather than global entropy reduction strategy. We will
empirically compare these methods in Section~\ref{sec:validity} for this
use case.

Perhaps the closest to our work on approximating arbitrary sources is that of
\citet{deoras2011variational}, where they used an RNN LM to generate samples
that they then used to train a $k$-gram LM.
In contrast, we use samples to approximate the joint state probability
$\gamma(q_s, q_a)$.
If $\ell$ is the average number of words per sentence and $N$ is the
total number of sampled sentences, then the time complexity of
\citet{deoras2011variational} is $\cO(N \ell |\Sigma|)$ as it takes
$\O(|\Sigma|)$ time to generate a sample from the neural model for every word.
This is the same as that of our algorithm in Section~\ref{sec:arbitrary_counting}.
However, our approach has several advantages. Firstly, for a given topology,
our algorithm provably finds the best KL minimized solution, whereas their
approach is optimal only when the size of the $k$-gram model tends to infinity.
Secondly, as we show in our experiments, the proposed algorithm performs better
compared to that of \citet{deoras2011variational}.

\citet{arisoy2014converting} and \citet{adel2014comparing} both
trained DNN models of different orders as the means to derive
probabilities for a $k$-gram model, rather than focusing on model
approximation as the above mentioned methods do.  Further, the methods
are applicable to $k$-gram models specifically, not the larger class of
target WFA to which our methods apply.

%% file: approx_experiments.tex
We now provide experimental evidence of the theory's validity and show its usefulness in
various applications. For the ease of notation,
we use \Wapprox~to denote the exact counting algorithm described in
Section \ref{sec:exact_counting} followed by the \textsc{KL-Minimization} algorithm
of Section \ref{sec:kl}.  Similarly,
we use \Wsampleapprox{N} to denote the sampled counting described in
Section \ref{sec:arbitrary_counting} with $N$ sampled sentences
followed by \textsc{KL-Minimization}.

We first give experimental evidence that supports the theory
in Section~\ref{sec:validity}. We then show how to approximate
neural models as WFAs in Section~\ref{sec:neural}. We also use the proposed method to
provide lower bounds on the perplexity given a target topology
in Section~\ref{sec:lower}.  Finally, we present experiments in
Section~\ref{sec:expWFAsource} addressing
scenarios with source WFAs, hence not requiring sampling. First, motivated by
low-memory applications such as (virtual) keyboard decoding~\cite{ouyang2017mobile}, we
use our approach to create compact language models in
Section~\ref{sec:small}. Then, in Section~\ref{sec:open}, we use our approach to create compact
open-vocabulary character language models from count-thresholded $n$-grams.

For most experiments (unless otherwise noted) we use the 1996 CSR Hub4 Language Model data,
LDC98T31 (English Broadcast News data). We use the processed form of the
corpus and further process it to downcase all the words and remove
punctuation. The resulting dataset has $132$M words in the
training set, $20$M words in the test set, and has $240$K unique
words. For all the experiments that use word models, we create a
vocabulary of approximately $32$K words consisting of all words
that appeared more than $50$ times in the training corpus. Using this
vocabulary, we create a trigram Katz model and prune it to contain
$2$M $n$-grams using entropy pruning \cite{stolcke2000entropy}. We use this
pruned model
as a baseline in all our word-based experiments. We use Katz smoothing
since it is amenable to pruning \cite{chelba2010study}. The perplexity
of this model on the test set is $144.4$.\footnote{For all perplexity
measurements we treat the unknown word as a single token instead of a
class.  To compute the perplexity with the unknown token being treated
as class, multiply the perplexity by $k^{0.0115}$, where $k$ is the
number of tokens in the unknown class and $0.0115$ is the out of
vocabulary rate in the test dataset.} All algorithms were implemented
using the open-source {\tt OpenFst} and {\tt OpenGrm} $n$-gram and stochastic
automata ({\tt SFst}) libraries\footnote{These libraries
are available at {\tt www.openfst.org} and {\tt www.opengrm.org}}
with the last library including these implementations
\cite{allauzen2007,roark2012,allauzen2018}.

\subsection{Empirical evidence of theory}
\label{sec:validity}
Recall that our goal is to find the distribution on a target DFA
topology that minimizes the KL divergence to the source
distribution. However, as stated in Section~\ref{sec:kl}, if the
target topology has failure transitions, the optimization objective is not
convex so the stationary point solution may not be
the global optimum. We now show that the model indeed converges
to a good solution in various cases empirically.

\paragraph{Idempotency}

When the target topology is the same as the source topology, we
show that the performance of the approximated model matches the
source model.   Let $p_s$ be the pruned Katz word model described
above. We  approximate $p_s$ onto the same topology using \Wapprox~and
\Wsampleapprox{$\cdot$} and then compute  perplexity on  the test
corpus.   The results are presented in  Figure~\ref{fig:idem}. The test
perplexity of the \Wapprox~model matches that of the source model
and the  performance of the  \Wsampleapprox{N} model approaches that of the
source model as the number of samples $N$ increases.
\begin{figure}[t]
\vspace{-3ex}
\begin{center}
\includegraphics[scale=.4]{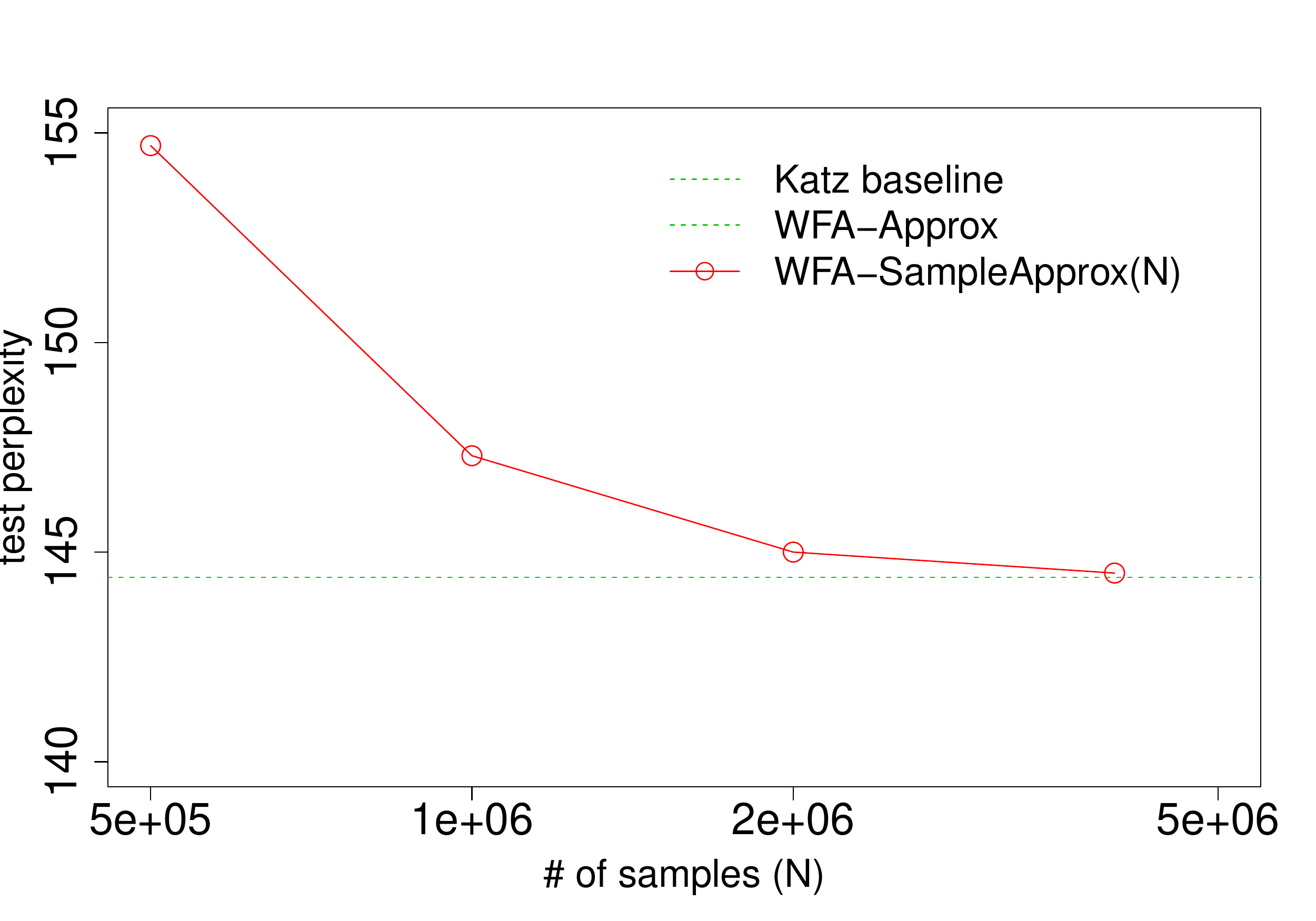}
\end{center}
\caption{{\em Idempotency:} Test set perplexity for Katz baseline and
approximations of that baseline
trained on the same data.  The Katz baseline and Katz
\Wapprox~plots are identical (and thus overlap), while \Wsampleapprox{N} plot converges to
the baseline.}
\label{fig:idem}
\end{figure}

\paragraph{Comparison to greedy pruning}

\begin{table}[t]
  \centering
  \caption{Test perplexity of greedy pruning and approximated models.}
\vskip .15in
  \begin{tabular}{c c c}
  Model size & Greedy pruning  & Approximated model  \\ \hline
  250K & $205.7$ & $198.3$\\
   500K & $177.3$& $173.0$ \\
  1M & $157.4$ & $155.7$\\
   1.5M & $149.0$ & $148.4$ \\
  \end{tabular}
\label{tab:greedy}
\end{table}

Recall that entropy pruning \cite{stolcke2000entropy} greedily removes
$n$-grams such that the KL divergence to the original model $p_s$ is
small. Let $p_{\text{greedy}}$ be the resulting model and
$A_\text{greedy}$ be the topology of $p_{\text{greedy}}$. If
the \textsc{KL-Minimization} converges to a good solution, then
approximating $p_s$ onto $A_\text{greedy}$ would give a model that is at
least as good as $p_{\text{greedy}}$. We show that this is indeed the
case; in fact, approximating $p_s$ onto $A_{\text{greedy}}$ performs
better than $p_{\text{greedy}}$. As before, let $p_s$ again be the $2M$
$n$-gram Katz model described above.  We
prune it to have $1$M $n$-grams and obtain $p_{\text{greedy}}$, which
has a test perplexity of $157.4$. We then approximate the source model $p_s$ on
$A_\text{greedy}$, the topology of $p_{\text{greedy}}$.
The resulting model has a test perplexity of
$155.7$, which is smaller than the test perplexity of
$p_{\text{greedy}}$. This shows that the approximation algorithm
indeed finds a good solution. We repeat the experiment with
different pruning thresholds and observe that the approximation algorithm
provides a good solution across the resulting model sizes. The results are in
Table~\ref{tab:greedy}.

\subsection{Neural models to WFA conversion}
\label{sec:neural}

Since neural models such as LSTMs give improved performance over $n$-gram
models, we investigate whether an LSTM distilled onto a
WFA model can obtain better performance than the baseline WFA trained
directly from Katz smoothing. As stated in the introduction, this could then
be used together with federated learning for fast and private on-device
inference, which was done in \citet{chen2019federated} using the
methods presented here.

To explore this, we train an LSTM language model on the training
data. The model has $2$ LSTM layers with $1024$ states and an embedding
size of $1024$. The resulting model has a test perplexity of $60.5$.
We approximate this model as a WFA in three ways.

\begin{figure}[t]
\vspace{-3ex}
\begin{center}
\includegraphics[scale=.4]{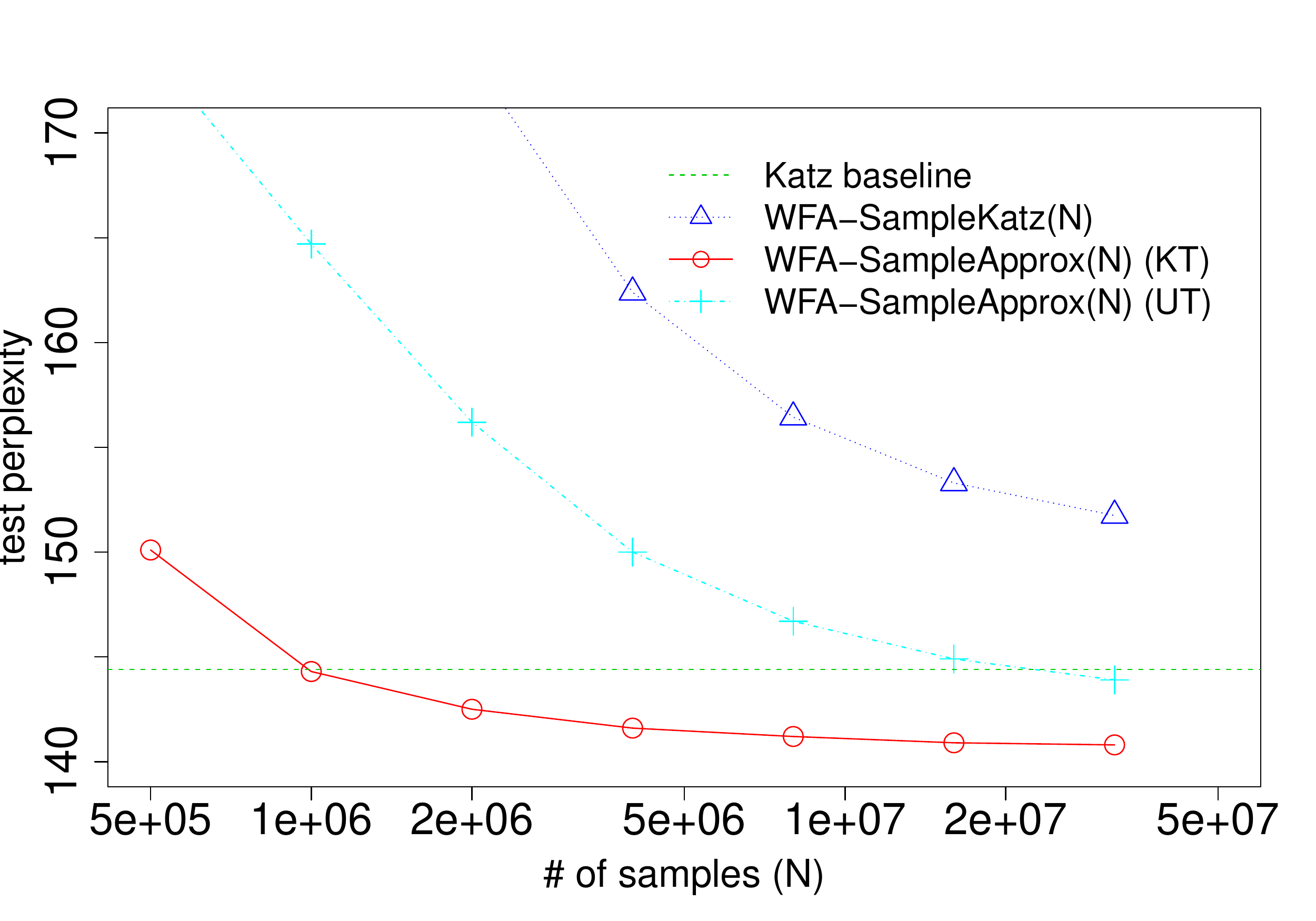}
\end{center}
\caption{{\em Approximated Neural Models for the English Broadcast News corpus:} Test set perplexity for Katz baseline and
LSTM models approximated in three ways. One uses LSTM samples to build a new Katz
model (\Wsamplekatz{N}). The remaining two use our approximation algorithm (\Wsampleapprox{N}), but
with different topologies. One topology ({\tt KT}) is known, using the baseline Katz topology,
and one ({\tt UT}) is unknown, using the samples drawn for \Wsampleapprox{N}
to also infer the topology.}
\label{fig:nn}
\end{figure}

The first way is to construct a Katz $n$-gram model on
$N$ LSTM samples and
entropy-prune to $2$M $n$-grams, which
we denote by \Wsamplekatz{N}. This approach is very similar to that of
~\cite{deoras2011variational}, except that we use Katz smoothing instead of
Kneser-Ney smoothing. We used Katz due to the fact that, as stated
earlier, Kneser-Ney models are not
amenable to pruning \cite{chelba2010study}. The second way is
to approximate onto the baseline Katz $2$M $n$-gram topology
described above using \Wsampleapprox{N}.  We refer to this experiment as
\Wsampleapprox{N} (KT), where KT stands for known topology.
The results are shown in
Figure \ref{fig:nn}. The \Wsamplekatz{N} models
perform significantly worse than the baseline Katz model even at
$32$M samples,
while the \Wsampleapprox{N} models
have better perplexity than the baseline Katz model with as little
as $1$M samples. With $32$M samples this way of approximating the
LSTM model as a WFA is $3.6$
better in perplexity than the baseline Katz.

Finally, if the WFA topology is
unknown we use the samples obtained in \Wsampleapprox{$\cdot$}
to create a Katz model entropy-pruned to $2$M $n$-grams.
We refer to this experiment as
\Wsampleapprox{N} (UT), where UT stands for unknown topology.
The results are shown in Figure~\ref{fig:nn}.
The approximated models obtained by \Wsampleapprox{N} (UT) perform better than
\Wsamplekatz{N}. However, they do not perform as well as \Wsampleapprox{N} (KT), the models
obtained with the known topology derived from the training data. However, with enough samples, their
performance is similar to that of the original Katz model.

We then compare our approximation method with other methods that approximate DNNs with $n$-gram LMs~\citep{arisoy2014converting, adel2014comparing}. Both these
methods require training DNNs of different orders and we used DNNs with
two layers with $1024$ nodes and an embedding size of $1024$. The results are in Table~\ref{tab:bnews_compare}.
The proposed algorithms \Wsampleapprox{$\cdot$} perform better than the existing approaches.

\begin{table}[t]
  \centering
  \caption{Test perplexity of $k$-gram models obtained by different approximation methods for the English Broadcast News corpus.
  We use 32M samples for both \Wsamplekatz{$\cdot$} and \Wsampleapprox{$\cdot$}.}
\vskip .15in
  \begin{tabular}{c c }
  Model & Test perplexity \\ \hline
  Katz & $144.4$ \\
  \Wsamplekatz{$\cdot$} &  $151.8$ \\
  \citet{arisoy2014converting} & $159.6$ \\
  \citet{adel2014comparing} & $146.7$  \\
  \Wsampleapprox{$\cdot$} (KT) &  $140.8$  \\
  \Wsampleapprox{$\cdot$} (UT) & $143.9$ \\
  \end{tabular}
\label{tab:bnews_compare}
\end{table}

\begin{table}[t]
  \centering
  \caption{Test perplexity of $k$-gram models obtained by different
  approaches for the Polish (Europarl) corpus.
  We use 32M samples for both \Wsamplekatz{$\cdot$} and \Wsampleapprox{$\cdot$}.}
\vskip .15in
  \begin{tabular}{c c }
  Model & Test perplexity \\ \hline
  Katz & $185.5$ \\
  \Wsamplekatz{$\cdot$} & $189.3$ \\
  \citet{arisoy2014converting} & $197.8$ \\
  \citet{adel2014comparing} & $187.1$  \\
  \Wsampleapprox{$\cdot$} (KT) &  $181.4$  \\
  \Wsampleapprox{$\cdot$} (UT) & $177.0$ \\
  \end{tabular}
\label{tab:europarl_compare}
\end{table}

\paragraph{Experiment on Polish}
To repeat the above experiment on a different language and corpus,
we turn to the Polish language
section\footnote{https://www.statmt.org/europarl/v7/pl-en.tgz} of the
Europarl corpus~\cite{koehn2005europarl}.
We chose Polish for this follow-up experiment due to the fact that it
belongs to a different language family than English, is relatively
highly inflected (unlike English) and is found in a publicly available corpus.
We use the processed form of the corpus and
further process it to downcase all the words and remove punctuation.
The resulting dataset has approximately $13$M words in the training
set and 210K words in the test set.
The selected vocabulary has
approximately 30K words, consisting of all words that appeared more
than $20$ times in the training corpus. Using this vocabulary,
we create a trigram Katz model and prune it to contain $2$M
$n$-grams using entropy pruning \cite{stolcke2000entropy}.
We use this pruned model as a baseline.
The results are in Figure~\ref{fig:euro_nn}. The trend is similar to that of
the English Broadcast News corpus with the proposed algorithm \Wsampleapprox{$\cdot$} performing
better than the other methods. We also compare the proposed algorithms with
other neural approximation algorithms. The comparison results are shown in Table~\ref{tab:europarl_compare}.

\begin{figure}[t]
\vspace{-3ex}
\begin{center}
\includegraphics[scale=.4]{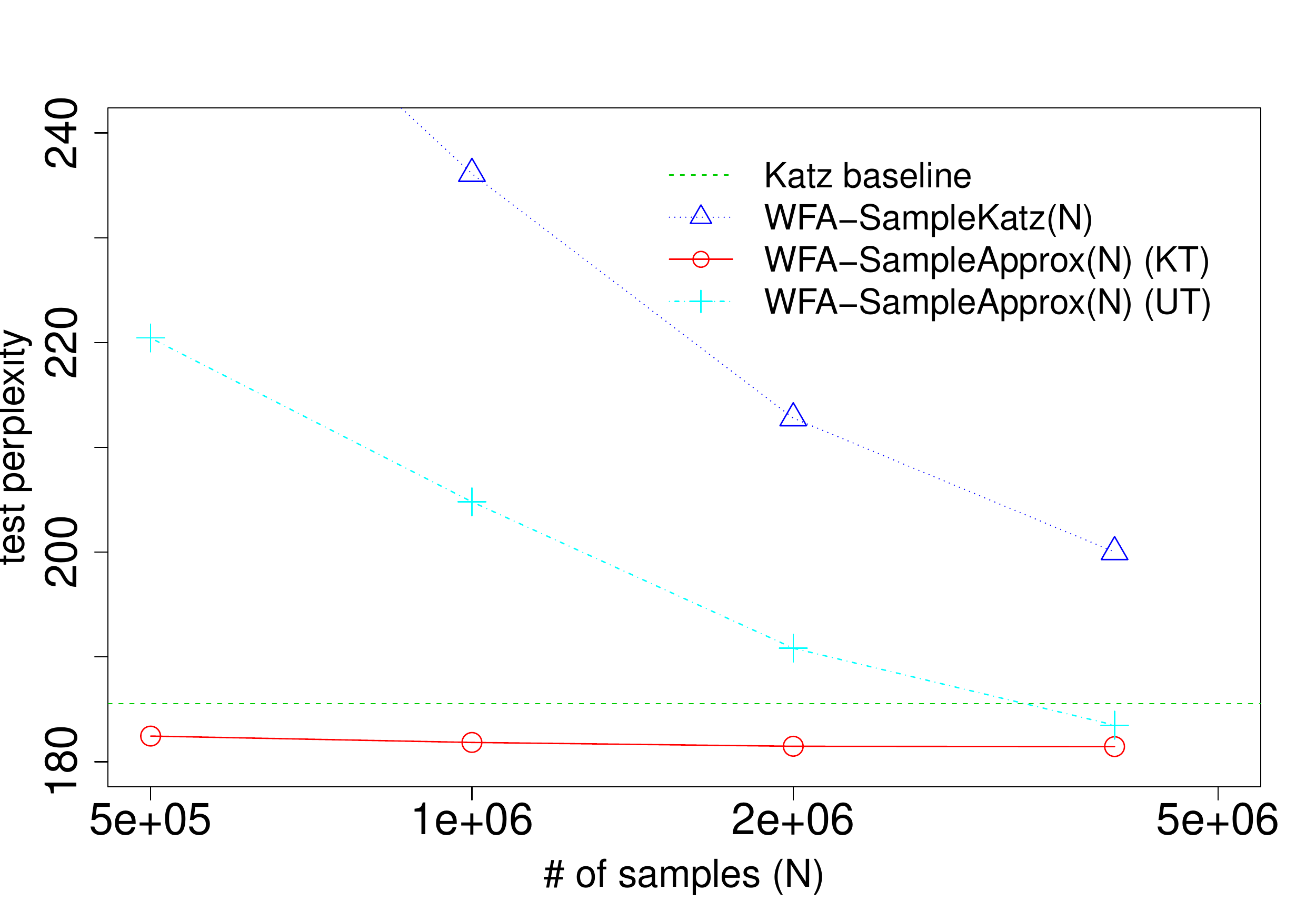}
\end{center}
\caption{{\em Approximated Neural Models for the Polish (Europarl) corpus:} Test set perplexity for Katz baseline and
LSTM models approximated in three ways. One uses LSTM samples to build a new Katz
model (\Wsamplekatz{N}). The remaining two use our approximation algorithm (\Wsampleapprox{N}), but
with different topologies. One topology ({\tt KT}) is known, using the baseline Katz topology,
and one ({\tt UT}) is unknown, using the samples drawn for \Wsampleapprox{N}
to also infer the topology.}
\label{fig:euro_nn}
\end{figure}

\subsection{Lower bounds on perplexity}
\label{sec:lower}

\paragraph{Best approximation for target topology}

The neural model in Section~\ref{sec:neural} has a perplexity of
$60.5$, but the best perplexity for the approximated model is $140.8$.
Is there a better approximation algorithm for the given target topology? We place
bounds on that next.

Let $T$ be the set of test
sentences.  The test-set log-perplexity of a model $p$ can be
written as
\[
\frac{1}{|T|} \sum_{x^* \in T} \log \frac{1}{p(x^*)} = \sum_{x^*} \hat{p}_t (x^*) \log \frac{1}{p(x^*)},
\]
where $\hat{p}_t$ is the empirical distribution of test
sentences. Observe that the best model with topology $A$ can
be computed as
\[
p'_a = \argmin_{p_a \in \sP(A)}  \sum_{x^*} \hat{p}_t (x^*) \log \frac{1}{p_a(x^*)},
\]
which is the model with topology $A$ that has minimal KL
divergence from the test distribution $\hat{p}_t$. This can be computed
using \Wapprox.  If we use this approach on the English Broadcast News test set
with the $2$M $n$-gram Katz model, the resulting model has perplexity of
$121.1$, showing that, under the assumption that the algorithm finds the global
KL divergence minimum, the test perplexity with this topology cannot be improved
beyond $121.1$, irrespective of the method.

\paragraph{Approximation onto best trigram topology}

What if we approximate the LSTM onto
the best trigram topology, how well does it perform over the test
data? To test this, we build a trigram model from the test data and
approximate the LSTM on the trigram topology. This approximated model has $11$M
$n$-grams and a perplexity of $81$. This shows that for large
datasets, the largest shortfall of $n$-gram models in the approximation
is due to the $n$-gram topology.

\subsection{WFA sources}
\label{sec:expWFAsource}
While distillation of neural models is an important use case for our
algorithms, there are other scenarios where approximations will be
derived from WFA sources.  In such cases, no sampling is required and
we can use the algorithms presented in Sections~\ref{sec:wfa_source}
or \ref{sec:exact_counting}, i.e., \Wapprox.  For example, it is not uncommon for
applications to place maximum size restrictions on models, so that
existing WFA models are too large and must be reduced in size prior to
use.  Section~\ref{sec:small} presents a couple of experiments focused
on character language model size reduction that were motivated by such
size restrictions.  Another scenario with WFA sources is when, for
privacy reasons, no language model training corpus is available in a
domain, but only minimum-count-thresholded (i.e., high frequency) word
$n$-gram counts are provided.  In Section~\ref{sec:open} we examine
the estimation of open-vocabulary character language models from such
data.

\subsubsection{Creating compact language models}
\label{sec:small}

\paragraph{Creating compact models for infrequent words}

In low-memory applications such as on-device keyboard
decoding~\cite{ouyang2017mobile}, it is often useful to have a character-level
WFA representation of a large set of vocabulary words that act only as unigrams,
e.g. those words beyond the $32$K words of our trigram model. We explore
how to compactly represent such a unigram-only model.

To demonstrate our approach, we take all the words in the training set
(without a count cut-off) and build a character-level deterministic WFA of those
words weighted by their unigram probabilities. This is
represented as a tree rooted at the initial state (a {\em trie}). This automaton has
$820$K transitions. Storing this many transitions can be prohibitive; we can reduce the
size in two steps.

The first step is to minimize this WFA using
weighted minimization \cite{mohri2000minimization}
to produce $p_{\text{char}}$, which has a topology $A_{\text{char}}$.
Although $p_{\text{char}}$ is already much smaller
(it has $378$K transitions, a $54\%$ reduction), we can go further by
approximating onto the minimal deterministic {\em unweighted} automaton,
$\text{Minimize}(A_{\text{char}})$. This gives us a model
with only $283$K transitions, a further $25\%$ reduction. Since
$\text{Minimize}(A_{\text{char}})$ accepts exactly the same words as
$A_{\text{char}}$, we are not corrupting our model by adding or removing any
vocabulary items. Instead we find an estimate which is as close as
possible to the original, but which is constrained to the minimal deterministic
representation that preserves the vocabulary.

To evaluate this approach, we randomly select a $20$K sentence
subset of the original test set, and represent each selected string
as a character-level sequence. We evaluate using cross entropy in bits-per-character,
common for character-level models. The resulting cross entropy for
$p_{\text{char}}$ is $1.557$ bits-per-character. By comparison,
the cross entropy for $p_{\text{char}}$ approximated onto
$\text{Minimize}(A_{\text{char}})$ is $1.560$ bits-per-character.
In exchange for this small
accuracy loss we are rewarded with a model which is $25\%$ smaller.
\citet{wolfsonkin2019latin} used the methods presented here to augment
the vocabulary of an on-device keyboard to deal with issues related to
a lack of standard orthography.

\paragraph{Creating compact WFA language models}

Motivated by the previous experiment, we also consider applying
(unweighted) minimization to $A_{\text{greedy}}$, the word-based
trigram topology that we pruned to $1$M $n$-grams described earlier.
In Table~\ref{tab:wordmin} we show that applying minimization to
$A_{\text{greedy}}$ and then approximating onto the resulting topology
leads to a reduction of $7\%$ in the number of transitions needed to
represent the model. However, the test perplexity also increases some.
To control for this, we prune the original model to a $1.08$M $n$-gram
topology $A'_\text{greedy}$ instead of the $1$M as before and apply
the same procedure to obtain an approximation on
$\text{Minimize}(A'_\text{greedy}$).  We achieve a $0.4\%$ perplexity
reduction compared to the approximation on $A_\text{greedy}$ with very
nearly the same number of transitions.

\begin{table}[t]
  \centering
  \caption{Test perplexity of models when approximated onto smaller topologies.}
\vskip .15in
  \begin{tabular}{c c c }
  Topology                              & Test perplexity   & \# Transitions \\ \hline
  $A_\text{greedy}$                  & $155.7$ & $1.13\text{M}$  \\
  $\text{Minimize}(A_\text{greedy}$) & $156.4$ & $1.05\text{M}$ \\
  \hline
  $A'_\text{greedy}$                  & $154.1$ & $1.22\text{M}$ \\
  $\text{Minimize}(A'_\text{greedy}$) & $154.9$ & $1.13\text{M}$ \\
  \end{tabular}
\label{tab:wordmin}
\end{table}

\subsubsection{Count thresholded data for privacy}
\label{sec:open}

One increasingly common scenario that can benefit from these
algorithms is modeling from frequency thresholded substring counts
rather than raw text.  For example, word $n$-grams
and their frequencies may be provided from certain domains of interest
only when they occur within at least $k$ separate documents.  With a
sufficiently large $k$ (say 50), no $n$-gram can be traced to a
specific document, thus providing privacy in the aggregation.  This is
known as $k$-anonymity \cite{samarati2001protecting}.

However, for any given domain, there are many kinds of models that one
may want to build depending on the task, some of which may be
trickier to estimate from such a collection of word $n$-gram counts
than with standard approaches for estimation from a given corpus.
For example, character $n$-gram models can be of high utility for tasks
like language identification, and have the benefit of a relatively
small memory footprint and low latency in use. However, character $n$-gram
models might be harder to learn from a $k$-anonymized corpus.

Here we will compare open-vocabulary character language models, which
accept all strings in $\Sigma^*$ for a character vocabulary $\Sigma$,
trained in several ways.  Each approach relies on the training corpus
and 32k vocabulary, with every out-of-vocabulary word replaced by a
single OOV symbol $\bigstar$.  Additionally, for each approach we add
50 to the unigram character count of any printable ASCII character, so
that even those that are unobserved in the words of our 32k vocabulary
have some observations.  Our three approaches are:

\begin{enumerate}
\item {\bf Baseline corpus trained models:} We count
character 5-grams from the $k$-anonymized corpus, then remove all $n$-grams that include
the $\bigstar$ symbol (in any position) prior to smoothing and
normalization. Here we present both Kneser-Ney and Witten Bell
smoothed models, as both are popular for character $n$-gram models.
\item {\bf Word trigram sampled model:} First we count word trigrams in the $k$-anonymized corpus
and discard any $n$-gram with the $\bigstar$ symbol (in any position) prior to
smoothing and normalization.  We then sample one million strings from a Katz smoothed
model and build a character 5-gram model from these strings.  We also use
this as our target topology for the next approach.
\item {\bf Word trigram KL minimization estimation:} We create a
source model by converting the 2M $n$-gram word trigram model into an
open vocabulary model. We do this using a specialized construction related to the
construction presented in Section 6 of \citet{chen2019federated}, briefly
described below, that converts the word model into a character
sequence model.  As this model is still closed vocabulary (see below),
we additionally smooth the unigram distribution with a character
trigram model trained from the words in the symbol table (and
including the 50 extra counts for every printable ASCII character as
with the other methods). From this source model, we estimate a model
on the sampled character 5-gram topology from the previous approach,
using our KL minimization algorithm.
\end{enumerate}

\paragraph{Converting word $n$-gram source to character sequence source model}

Briefly, for every state $q$ in the $n$-gram automaton, the set of words
labeling transitions leaving $q$ are represented as a trie of characters
including a final end-of-word symbol.  Each resulting transition labeled with
the end-of-word symbol represents the last transition for that particular
word spelled out by that sequence of transitions, hence is assigned the same
next state as the original word transition.  If $q$ has a backoff transition
pointing to its backoff state $q'$, then each new internal state in the
character trie backs off to the corresponding state in the
character trie leaving $q'$.  The presence of the corresponding state
in the character trie leaving $q'$ is guaranteed because the $n$-gram
automaton is backoff-complete, as discussed in Section \ref{subsec:phiwfa}.

As stated above, this construction converts from word sequences to
character sequences, but will only accept character sequences
consisting of strings of in-vocabulary words, i.e., this is still
closed vocabulary.  To make it open vocabulary, we
further backoff the character trie leaving the unigram state to a
character $n$-gram model estimated from the symbol table (and additional
ASCII character observations).  This is done using a very similar
construction to that described above.  The resulting model is used as
the source model for our KL minimization algorithm, to estimate a
distribution over the sampled character 5-gram topology.

We encode the test set as a sequence of characters, without using the
symbol table since our models are intended to be open vocabulary.
Following typical practice for open-vocabulary settings, we evaluate with
bits-per-character. The results are presented in Table \ref{tab:open}.
Here we achieve slightly lower bits-per-character than even what we
get straight from the corpus, perhaps due to better regularization of
the word-based model than with either Witten-Bell or Kneser-Ney on the
character n-grams.

\begin{table}[t]
  \centering
  \caption{Comparison of character 5-gram models derived from either
  the original corpus or a word trigram model.  Size of the models is
  presented in terms of the number of character $n$-grams, the numbers
  of states and transitions in the automaton representation, and the file size in
  MB.  The two corpus estimated models have the same topology, hence
  the same size; as do the two word trigram estimated models.}\label{tab:open}
\vskip .15in
\begin{tabular}{lcccccc}
& $n$-grams & states & transitions & & & \\
Source & (x1000) & (x1000) & (x1000) & MB & Estimation &
bits/char\\\hline
Corpus & 336 & 60 & 381 & 6.5 & Kneser-Ney & 2.04\\
 &  &  &  &  & Witten-Bell (WB) & 2.01\\\hline
Word trigram & 277 & 56 & 322 & 5.6 & Sampled (WB) & 2.36\\
 &  &  &  &  & KL min & 1.99
\end{tabular}
\end{table}

%% file: approx_lib.tex
All of the algorithms presented here are available
in the {\tt OpenGrm} libraries at {\small \url{http://www.opengrm.org}}
under the topic: {\em SFST Library: operations to normalize, sample,
combine, and approximate stochastic finite-state transducers.}
We illustrate the use of this library by showing how to implement
some of the experiments in the previous section.

\subsection{Example data and models}

Instead of Broadcast News, we will use the text of Oscar Wilde's {\em The Importance of Being
Earnest} for our tutorial example. This is a small tutorial corpus that we
make available at {\small \url{http://sfst.opengrm.org}}.

This corpus of approximately 1688 sentences and 18,000-words
has been upper-cased and had punctuation removed. The
first 850 sentences were used as training data and the remaining 838 sentences
used as test data. From these, we produce two 1000-word vocabulary Katz bigram models,
the $\sim\! 6k$ $n$-gram {\tt earnest\_train.mod} and the $\sim\! 4k$ $n$-gram
{\tt earnest\_test.mod}.
We also used relative-entropy pruning to create the $\sim\! 2k$
$n$-gram {\tt earnest\_train.pru}. The data, the steps to generate these models and how
to compute their perplexity using
{\em OpenGrm NGram} are all fully detailed in the {\tt QuickTour} topic
at {\small \url{http://sfst.opengrm.org}}.

\subsection{Computing the approximation}

The following step shows how to compute the approximation of a $\varphi$-WFA
model onto a $\varphi$-WFA topology. In the example below,
the first argument, {\tt earnest\_train.mod}, is the source model,
and the second argument, {\tt earnest\_train.pru}, provides the topology.
The result is a $\varphi$-WFA whose perplexity can be measured as before.

\begin{center}
{\tt
\begin{tabular}{l}
\$ sfstapprox -phi\_label=0 earnest\_train.mod earnest\_train.pru \textbackslash\\
~~>earnest\_train.approx
\end{tabular}
}
\end{center}

An alternative, equivalent way to perform this approximation is to break it into two steps,
with the counting and normalization (KL divergence minimization) done separately.

\begin{center}
{\tt
\begin{tabular}{l}
\$ sfstcount -phi\_label=0 earnest\_train.mod earnest\_train.pru \textbackslash\\
~~>earnest\_train.approx\_cnts\\
\$ sfstnormalize -method=kl\_min -phi\_label=0 \textbackslash\\
~~earnest\_train.approx\_cnts >earnest\_train.approx
\end{tabular}
}
\end{center}

We can now use these utilities to run some experiments analogous to our larger Broadcast
News experiments by using different target topologies. The results are shown in
Table~\ref{tab:tutres}. We see in the idempotency experiment, the perplexity of the
approximation on the same topology matches the source. In the greedy-pruning experiment,
the approximation onto the greedy-pruned topology yields a better perplexity than the
greedily-pruned model. Finally, the approximation onto the test-set bigram topology gives a
better perplexity than the training-set model since we include all the relevant bigrams.

\begin{table}
\centering
\caption{Perplexities for example experiments.}
\begin{tabular}{l|l|l|l}
\multirow{2}{*}{\bf Experiment } & \multirow{2}{1cm}{\bf Source } & \multirow{2}{*}{\bf Topology }
& \multirow{2}{1cm}{\bf Approx. }\\
& {\bf Model} & & {\bf Perplexity}\\ \hline
Idempotency & earnest\_train.mod & earnest\_train.mod & 73.41 \\
Comparison to greedy pruning & earnest\_train.mod & earnest\_train.pru & 83.12 \\
Approx.~onto best bigram topology  & earnest\_train.mod & earnest\_test.pru & 69.68 \\
\end{tabular}
\label{tab:tutres}
\end{table}

\subsection{Available operations}
Table~\ref{tab:sfstops} lists the available command-line operations in the {\tt OpenGrm SFst} library.
We show command-line utilities here;
there are corresponding C++ library functions that can be called from within a program; see
{\small \url{http://sfst.opengrm.org}}.

\begin{table}
\centering
\caption{Available Operations in the OpenGrm SFst Library}
\vskip .15in
{\tt
\begin{tabular}{@{}l|p{2.6in}}
{\bf Operation } & {\bf Description}\\ \hline
sfstapprox & Approximates a stochastic $\varphi$-WFA  \\
 & ~~wrt a backoff-complete $\varphi$-WFA.\\
sfstcount & Counts from a stochastic $\varphi$-WFA  \\
& ~~wrt a backoff-complete $\varphi$-WFA.\\
sfstintersect & Intersects two $\varphi$-WFAs.\\~\\
sfstnormalize -method=global & Globally normalizes a $\varphi$-WFA.\\~\\
sfstnormalize -method=kl\_min & Normalizes a count $\varphi$-WFA using \\
& ~~KL divergence minimization.\\
sfstnormalize -method=local & Locally normalizes a $\varphi$-WFA.\\~\\
sfstnormalize -method=phi & $\varphi$-normalizes a $\varphi$-WFA.\\~\\
sfstperplexity & Computes perplexity of a \\
& ~~stochastic $\varphi$-WFA.\\
sfstrandgen & Randomly generates paths from a\\
& ~~stochastic $\varphi$-WFA.\\
sfstshortestdistance & Computes the shortest distance \\
& ~~on a $\varphi$-WFA.\\
sfsttrim & Removes useless states and \\
& ~~transitions in a $\varphi$-WFA.
\end{tabular}
}
\label{tab:sfstops}
\end{table}

%% file: approx_discuss.tex
In this paper, we have presented an algorithm for minimizing the
KL-divergence between a probabilistic source model over sequences and
a WFA target model.  Our algorithm is general enough to permit source
models of arbitrary form (e.g., RNNs) and a wide class of target WFA
models, importantly including those with failure transitions, such as $n$-gram models.  We
provide some experimental validation of our algorithms, including
demonstrating that it is well-behaved in common scenarios and that it
yields improved performance over baseline $n$-gram models using the same
WFA topology.  Additionally, we use our methods to provide lower
bounds on how well a given WFA topology can model a given test set.
All of the algorithms reported here are available in the open-source
{\tt OpenGrm} libraries at {\small \url{http://opengrm.org}}.

In addition to the above-mentioned results, we also demonstrated that
optimizing the WFA topology for the given test set yields far better
perplexities than were obtained using WFA topologies derived from
training data alone, suggesting that the problem of deriving an
appropriate WFA topology -- something we do not really touch on in
this paper -- is particularly important.

%% file: approx.bbl
\begin{thebibliography}{60}
\providecommand{\natexlab}[1]{#1}
\providecommand{\url}[1]{\texttt{#1}}
\expandafter\ifx\csname urlstyle\endcsname\relax
  \providecommand{\doi}[1]{doi: #1}\else
  \providecommand{\doi}{doi: \begingroup \urlstyle{rm}\Url}\fi

\bibitem[Adel et~al.(2014)Adel, Kirchhoff, Vu, Telaar, and
  Schultz]{adel2014comparing}
Heike Adel, Katrin Kirchhoff, Ngoc~Thang Vu, Dominic Telaar, and Tanja Schultz.
\newblock Comparing approaches to convert recurrent neural networks into
  backoff language models for efficient decoding.
\newblock In \emph{Fifteenth Annual Conference of the International Speech
  Communication Association (Interspeech)}, 2014.

\bibitem[Aho and Corasick(1975)]{aho1975}
Alfred~V. Aho and Margaret~J. Corasick.
\newblock Efficient string matching: an aid to bibliographic search.
\newblock \emph{Communications of the ACM}, 18\penalty0 (6):\penalty0 333--340,
  1975.

\bibitem[Albert and Kari(2009)]{AlbertKari2009}
J{\"u}rgen Albert and Jarkko Kari.
\newblock Digital image compression.
\newblock In \emph{Handbook of weighted automata}. Springer, 2009.

\bibitem[Allauzen and Riley(2018)]{allauzen2018}
Cyril Allauzen and Michael~D. Riley.
\newblock Algorithms for weighted finite automata with failure transitions.
\newblock In \emph{International Conference on Implementation and Application
  of Automata}, pages 46--58. Springer, 2018.

\bibitem[Allauzen et~al.(2003)Allauzen, Mohri, and Roark]{Allauzen03}
Cyril Allauzen, Mehryar Mohri, and Brian Roark.
\newblock Generalized algorithms for constructing language models.
\newblock In \emph{Proceedings of ACL}, pages 40--47, 2003.

\bibitem[Allauzen et~al.(2007)Allauzen, Riley, Schalkwyk, Skut, and
  Mohri]{allauzen2007}
Cyril Allauzen, Michael Riley, Johan Schalkwyk, Wojciech Skut, and Mehryar
  Mohri.
\newblock {OpenFst Library}.
\newblock \emph{http://www.openfst.org}, 2007.

\bibitem[Angluin(1987)]{angluin1987}
Dana Angluin.
\newblock Learning regular sets from queries and counterexamples.
\newblock \emph{Information and computation}, 75\penalty0 (2):\penalty0
  87--106, 1987.

\bibitem[Angluin(1988)]{angluin1988}
Dana Angluin.
\newblock Identifying languages from stochastic examples.
\newblock Technical Report YALEU /DCS /RR-614, Yale University, 1988.

\bibitem[Arisoy et~al.(2014)Arisoy, Chen, Ramabhadran, and
  Sethy]{arisoy2014converting}
Ebru Arisoy, Stanley~F. Chen, Bhuvana Ramabhadran, and Abhinav Sethy.
\newblock Converting neural network language models into back-off language
  models for efficient decoding in automatic speech recognition.
\newblock \emph{IEEE/ACM Transactions on Audio, Speech and Language Processing
  (TASLP)}, 22\penalty0 (1):\penalty0 184--192, 2014.

\bibitem[Balle and Mohri(2012)]{Balle2012}
Borja Balle and Mehryar Mohri.
\newblock Spectral learning of general weighted automata via constrained matrix
  completion.
\newblock In \emph{Advances in neural information processing systems}, pages
  2159--2167, 2012.

\bibitem[Balle et~al.(2014)Balle, Carreras, Luque, and Quattoni]{Balle2014}
Borja Balle, Xavier Carreras, Franco~M. Luque, and Ariadna Quattoni.
\newblock Spectral learning of weighted automata.
\newblock \emph{Machine learning}, 96\penalty0 (1-2):\penalty0 33--63, 2014.

\bibitem[Breuel(2008)]{breuel2008}
Thomas~M. Breuel.
\newblock The {OCR}opus open source {OCR} system.
\newblock In \emph{Proceedings of IS\&T/SPIE 20th Annual Symposium}, 2008.

\bibitem[Carrasco(1997)]{carrasco1997}
Rafael~C. Carrasco.
\newblock Accurate computation of the relative entropy between stochastic
  regular grammars.
\newblock \emph{RAIRO-Theoretical Informatics and Applications}, 31\penalty0
  (5):\penalty0 437--444, 1997.

\bibitem[Carrasco and Oncina(1994)]{carrasco1994}
Rafael~C. Carrasco and Jos{\'e} Oncina.
\newblock Learning stochastic regular grammars by means of a state merging
  method.
\newblock In \emph{International Colloquium on Grammatical Inference}, pages
  139--152. Springer, 1994.

\bibitem[Carrasco and Oncina(1999)]{carrasco1999}
Rafael~C. Carrasco and Jose Oncina.
\newblock Learning deterministic regular grammars from stochastic samples in
  polynomial time.
\newblock \emph{RAIRO-Theoretical Informatics and Applications}, 33\penalty0
  (1):\penalty0 1--19, 1999.

\bibitem[Chelba et~al.(2010)Chelba, Brants, Neveitt, and Xu]{chelba2010study}
Ciprian Chelba, Thorsten Brants, Will Neveitt, and Peng Xu.
\newblock Study on interaction between entropy pruning and {Kneser-Ney}
  smoothing.
\newblock In \emph{Eleventh Annual Conference of the International Speech
  Communication Association (Interspeech)}, 2010.

\bibitem[Chen et~al.(2019)Chen, Suresh, Mathews, Wong, Beaufays, Allauzen, and
  Riley]{chen2019federated}
Mingqing Chen, Ananda~Theertha Suresh, Rajiv Mathews, Adeline Wong,
  Fran{\c{c}}oise Beaufays, Cyril Allauzen, and Michael Riley.
\newblock Federated learning of {N}-gram language models.
\newblock In \emph{Proceedings of the 23rd Conference on Computational Natural
  Language Learning (CoNLL)}, 2019.

\bibitem[Chen and Goodman(1998)]{Chen98}
Stanley Chen and Joshua Goodman.
\newblock An empirical study of smoothing techniques for language modeling.
\newblock Technical report, TR-10-98, Harvard University, 1998.

\bibitem[Cicchello and Kremer(2003)]{cicchello2003}
Orlando Cicchello and Stefan~C. Kremer.
\newblock Inducing grammars from sparse data sets: a survey of algorithms and
  results.
\newblock \emph{Journal of Machine Learning Research}, 4\penalty0
  (Oct):\penalty0 603--632, 2003.

\bibitem[Cortes et~al.(2008)Cortes, Mohri, Rastogi, and Riley]{cortes2008}
Corinna Cortes, Mehryar Mohri, Ashish Rastogi, and Michael Riley.
\newblock On the computation of the relative entropy of probabilistic automata.
\newblock \emph{International Journal of Foundations of Computer Science},
  19\penalty0 (01):\penalty0 219--242, 2008.

\bibitem[Dempster et~al.(1977)Dempster, Laird, and Rubin]{dempster1977}
Arthur~P. Dempster, Nan~M. Laird, and Donald~B. Rubin.
\newblock Maximum likelihood from incomplete data via the {EM} algorithm.
\newblock \emph{Journal of the royal statistical society. Series B
  (methodological)}, pages 1--38, 1977.

\bibitem[Deoras et~al.(2011)Deoras, Mikolov, Kombrink, Karafi{\'a}t, and
  Khudanpur]{deoras2011variational}
Anoop Deoras, Tom{\'a}{\v{s}} Mikolov, Stefan Kombrink, Martin Karafi{\'a}t,
  and Sanjeev Khudanpur.
\newblock Variational approximation of long-span language models for {LVCSR}.
\newblock In \emph{Acoustics, Speech and Signal Processing (ICASSP), 2011 IEEE
  International Conference on}, pages 5532--5535. IEEE, 2011.

\bibitem[Dupont(1996)]{dupont1996}
Pierre Dupont.
\newblock Incremental regular inference.
\newblock In \emph{International Colloquium on Grammatical Inference}, pages
  222--237. Springer, 1996.

\bibitem[Durbin et~al.(1998)Durbin, Eddy, Krogh, and Mitchison]{Durbin1998}
Richard Durbin, Sean~R. Eddy, Anders Krogh, and Graeme~J. Mitchison.
\newblock \emph{Biological Sequence Analysis: Probabilistic Models of Proteins
  and Nucleic Acids}.
\newblock Camb. Univ. Press, 1998.

\bibitem[Ebden and Sproat(2015)]{rws2015}
Peter Ebden and Richard Sproat.
\newblock The {Kestrel} {TTS} text normalization system.
\newblock \emph{Natural Language Engineering}, 21\penalty0 (3):\penalty0
  333--353, 2015.

\bibitem[Eisner(2001)]{eisner2001}
Jason Eisner.
\newblock Expectation semirings: Flexible {EM} for learning finite-state
  transducers.
\newblock In \emph{Proceedings of the ESSLLI workshop on finite-state methods
  in NLP}, pages 1--5, 2001.

\bibitem[Giles et~al.(1992)Giles, Miller, Chen, Chen, Sun, and Lee]{giles1992}
C.~Lee Giles, Clifford~B. Miller, Dong Chen, Hsing-Hen Chen, Guo-Zheng Sun, and
  Yee-Chun Lee.
\newblock Learning and extracting finite state automata with second-order
  recurrent neural networks.
\newblock \emph{Neural Computation}, 4\penalty0 (3):\penalty0 393--405, 1992.

\bibitem[Gold(1967)]{gold1967}
E.~Mark Gold.
\newblock Language identification in the limit.
\newblock \emph{Information and control}, 10\penalty0 (5):\penalty0 447--474,
  1967.

\bibitem[Gold(1978)]{gold1978}
E.~Mark Gold.
\newblock Complexity of automaton identification from given data.
\newblock \emph{Information and control}, 37\penalty0 (3):\penalty0 302--320,
  1978.

\bibitem[Hard et~al.(2018)Hard, Rao, Mathews, Beaufays, Augenstein, Eichner,
  Kiddon, and Ramage]{hard2018federated}
Andrew Hard, Kanishka Rao, Rajiv Mathews, Fran{\c{c}}oise Beaufays, Sean
  Augenstein, Hubert Eichner, Chlo{\'e} Kiddon, and Daniel Ramage.
\newblock Federated learning for mobile keyboard prediction.
\newblock \emph{arXiv preprint arXiv:1811.03604}, 2018.

\bibitem[Hellsten et~al.(2017)Hellsten, Roark, Goyal, Allauzen, Beaufays,
  Ouyang, Riley, and Rybach]{hellsten17}
Lars Hellsten, Brian Roark, Prasoon Goyal, Cyril Allauzen, Fran{\c{c}}oise
  Beaufays, Tom Ouyang, Michael Riley, and David Rybach.
\newblock Transliterated mobile keyboard input via weighted finite-state
  transducers.
\newblock In \emph{FSMNLP 2017}, pages 10--19, 2017.

\bibitem[Horst and Thoai(1999)]{horst1999}
Reiner Horst and Nguyen~V. Thoai.
\newblock {DC} programming: overview.
\newblock \emph{Journal of Optimization Theory and Applications}, 103\penalty0
  (1):\penalty0 1--43, 1999.

\bibitem[Iglesias et~al.(2011)Iglesias, Allauzen, Byrne, de~Gispert, and
  Riley]{iglesias2011}
Gonzalo Iglesias, Cyril Allauzen, William Byrne, Adri{\`a} de~Gispert, and
  Michael Riley.
\newblock Hierarchical phrase-based translation representations.
\newblock In \emph{EMNLP 2011}, pages 1373--1383, 2011.

\bibitem[Jacobsson(2005)]{jacobsson2005}
Henrik Jacobsson.
\newblock Rule extraction from recurrent neural networks: A taxonomy and
  review.
\newblock \emph{Neural Computation}, 17\penalty0 (6):\penalty0 1223--1263,
  2005.

\bibitem[Katz(1987)]{katz}
Slava~M. Katz.
\newblock Estimation of probabilities from sparse data for the language model
  component of a speech recogniser.
\newblock \emph{IEEE Transactions on Acoustic, Speech, and Signal Processing},
  35\penalty0 (3):\penalty0 400--401, 1987.

\bibitem[Koehn(2005)]{koehn2005europarl}
Philipp Koehn.
\newblock Europarl: A parallel corpus for statistical machine translation.
\newblock In \emph{MT summit}, volume~5, pages 79--86. Citeseer, 2005.

\bibitem[Kone{\v{c}}n{\`y} et~al.(2016)Kone{\v{c}}n{\`y}, McMahan, Yu,
  Richt{\'a}rik, Suresh, and Bacon]{konevcny2016federated}
Jakub Kone{\v{c}}n{\`y}, H.~Brendan McMahan, Felix~X. Yu, Peter Richt{\'a}rik,
  Ananda~Theertha Suresh, and Dave Bacon.
\newblock Federated learning: Strategies for improving communication
  efficiency.
\newblock \emph{arXiv preprint arXiv:1610.05492}, 2016.

\bibitem[Lecorv{\'e} and Motlicek(2012)]{lecorve2012conversion}
Gw{\'e}nol{\'e} Lecorv{\'e} and Petr Motlicek.
\newblock Conversion of recurrent neural network language models to weighted
  finite state transducers for automatic speech recognition.
\newblock In \emph{Thirteenth Annual Conference of the International Speech
  Communication Association (Interspeech)}, 2012.

\bibitem[McMahan et~al.(2017)McMahan, Moore, Ramage, Hampson, and
  y~Arcas]{mcmahan2017communication}
Brendan McMahan, Eider Moore, Daniel Ramage, Seth Hampson, and Blaise~Aguera
  y~Arcas.
\newblock Communication-efficient learning of deep networks from decentralized
  data.
\newblock In \emph{Artificial Intelligence and Statistics}, pages 1273--1282.
  PMLR, 2017.

\bibitem[Mohri(1997)]{mohri1997}
Mehryar Mohri.
\newblock String-matching with automata.
\newblock \emph{Nord. J. Comput.}, 4\penalty0 (2):\penalty0 217--231, 1997.

\bibitem[Mohri(2000)]{mohri2000minimization}
Mehryar Mohri.
\newblock Minimization algorithms for sequential transducers.
\newblock \emph{Theoretical Computer Science}, 234\penalty0 (1-2):\penalty0
  177--201, 2000.

\bibitem[Mohri(2002)]{mohri2002}
Mehryar Mohri.
\newblock Semiring frameworks and algorithms for shortest-distance problems.
\newblock \emph{Journal of Automata, Languages and Combinatorics}, 7\penalty0
  (3):\penalty0 321--350, 2002.

\bibitem[Mohri(2009)]{Mohri2009}
Mehryar Mohri.
\newblock Weighted automata algorithms.
\newblock In \emph{Handbook of Weighted Automata}, pages 213--254. Springer,
  2009.

\bibitem[Mohri et~al.(2008)Mohri, Pereira, and Riley]{MohriPereiraRiley2008}
Mehryar Mohri, Fernando C.~N. Pereira, and Michael Riley.
\newblock Speech recognition with weighted finite-state transducers.
\newblock In \emph{Handbook on speech proc.\ and speech comm.} Springer, 2008.

\bibitem[Novak et~al.(2013)Novak, Minematsu, and Hirose]{novak2013}
Josef~R. Novak, Nobuaki Minematsu, and Keikichi Hirose.
\newblock Failure transitions for joint n-gram models and g2p conversion.
\newblock In \emph{Fourteenth Annual Conference of the International Speech
  Communication Association (Interspeech)}, pages 1821--1825, 2013.

\bibitem[Okudono et~al.(2020)Okudono, Waga, Sekiyama, and
  Hasuo]{okudono2019weighted}
Takamasa Okudono, Masaki Waga, Taro Sekiyama, and Ichiro Hasuo.
\newblock Weighted automata extraction from recurrent neural networks via
  regression on state spaces.
\newblock In \emph{Proceedings of the AAAI Conference on Artificial
  Intelligence}, volume~34, pages 5306--5314, 2020.

\bibitem[Oncina and Garcia(1992)]{oncina1992}
Jos{\'e} Oncina and Pedro Garcia.
\newblock Identifying regular languages in polynomial time.
\newblock In \emph{Advances in Structural and Syntactic Pattern Recognition},
  pages 99--108. World Scientific, 1992.

\bibitem[Ouyang et~al.(2017)Ouyang, Rybach, Beaufays, and
  Riley]{ouyang2017mobile}
Tom Ouyang, David Rybach, Fran{\c{c}}oise Beaufays, and Michael Riley.
\newblock Mobile keyboard input decoding with finite-state transducers.
\newblock \emph{arXiv preprint arXiv:1704.03987}, 2017.

\bibitem[Parekh and Honavar(2000)]{parekh2000}
Rajesh Parekh and Vasant Honavar.
\newblock Grammar inference, automata induction, and language acquisition.
\newblock \emph{Handbook of natural language processing}, pages 727--764, 2000.

\bibitem[Pitt(1989)]{pitt1989}
Leonard Pitt.
\newblock Inductive inference, {DFAs}, and computational complexity.
\newblock In \emph{International Workshop on Analogical and Inductive
  Inference}, pages 18--44. Springer, 1989.

\bibitem[Roark et~al.(2012)Roark, Sproat, Allauzen, Riley, Sorensen, and
  Tai]{roark2012}
Brian Roark, Richard Sproat, Cyril Allauzen, Michael Riley, Jeffrey Sorensen,
  and Terry Tai.
\newblock The {OpenGrm} open-source finite-state grammar software libraries.
\newblock \emph{Proceedings of the ACL 2012 System Demonstrations}, pages
  61--66, 2012.

\bibitem[Samarati(2001)]{samarati2001protecting}
Pierangela Samarati.
\newblock Protecting respondents identities in microdata release.
\newblock \emph{IEEE transactions on Knowledge and Data Engineering},
  13\penalty0 (6):\penalty0 1010--1027, 2001.

\bibitem[Sriperumbudur and Lanckriet(2009)]{sriperumbudur2009convergence}
Bharath~K. Sriperumbudur and Gert~R.G. Lanckriet.
\newblock On the convergence of the concave-convex procedure.
\newblock In \emph{Proceedings of the 22nd International Conference on Neural
  Information Processing Systems}, pages 1759--1767. Curran Associates Inc.,
  2009.

\bibitem[Stolcke(2000)]{stolcke2000entropy}
Andreas Stolcke.
\newblock Entropy-based pruning of backoff language models.
\newblock \emph{arXiv preprint cs/0006025}, 2000.

\bibitem[Sundermeyer et~al.(2012)Sundermeyer, Schl{\"u}ter, and
  Ney]{sundermeyer2012lstm}
Martin Sundermeyer, Ralf Schl{\"u}ter, and Hermann Ney.
\newblock {LSTM} neural networks for language modeling.
\newblock In \emph{Thirteenth annual conference of the international speech
  communication association}, 2012.

\bibitem[Suresh et~al.(2019)Suresh, Roark, Riley, and
  Schogol]{suresh2019fsmnlp}
Ananda~Theertha Suresh, Brian Roark, Michael Riley, and Vlad Schogol.
\newblock Distilling weighted finite automata from arbitrary probabilistic
  models.
\newblock In \emph{Proceedings of the 14th International Conference on
  Finite-State Methods and Natural Language Processing (FSMNLP)}, 2019.

\bibitem[Ti\~no and Vojtek(1997)]{tino1997}
Peter Ti\~no and Vladimir Vojtek.
\newblock Extracting stochastic machines from recurrent neural networks trained
  on complex symbolic sequences.
\newblock In \emph{Knowledge-Based Intelligent Electronic Systems, 1997.
  KES'97. Proceedings., 1997 First International Conference on}, volume~2,
  pages 551--558. IEEE, 1997.

\bibitem[Weiss et~al.(2018)Weiss, Goldberg, and Yahav]{weiss2018extracting}
Gail Weiss, Yoav Goldberg, and Eran Yahav.
\newblock Extracting automata from recurrent neural networks using queries and
  counterexamples.
\newblock In \emph{International Conference on Machine Learning}, pages
  5244--5253, 2018.

\bibitem[Weiss et~al.(2019)Weiss, Goldberg, and Yahav]{weiss2019learning}
Gail Weiss, Yoav Goldberg, and Eran Yahav.
\newblock Learning deterministic weighted automata with queries and
  counterexamples.
\newblock In \emph{Advances in Neural Information Processing Systems}, pages
  8560--8571, 2019.

\bibitem[Wolf-Sonkin et~al.(2019)Wolf-Sonkin, Schogol, Roark, and
  Riley]{wolfsonkin2019latin}
Lawrence Wolf-Sonkin, Vlad Schogol, Brian Roark, and Michael Riley.
\newblock Latin script keyboards for south asian languages with finite-state
  normalization.
\newblock In \emph{Proceedings of the 14th International Conference on
  Finite-State Methods and Natural Language Processing (FSMNLP)}, 2019.

\end{thebibliography}
